%% file: main.tex
\documentclass{article}
\usepackage{graphicx} 
\usepackage{subfigure}
\usepackage[left=1.4in, right=1.4in, top=1.3in, bottom=1.2in]{geometry}

\usepackage{authblk}
\usepackage[round]{natbib}
\usepackage{amsmath, amssymb, amsthm}
\usepackage{booktabs}
\usepackage{caption}
\usepackage{algorithm}
\usepackage{algpseudocode}
\usepackage{float}

\usepackage{hyperref}
\hypersetup{colorlinks,
linkcolor=blue,
citecolor=blue,
urlcolor=magenta,
linktocpage,
plainpages=false}

\input{defs.tex}

\title{Fault Tolerant ML: Efficient Meta-Aggregation and Synchronous Training}
\author{Tehila Dahan}
\author{Kfir Y. Levy}
\affil{Technion - Israel Institute of Technology, Haifa, Israel}

\date{}
\begin{document}

\maketitle

\begin{abstract}%
In this paper, we investigate the challenging framework of Byzantine-robust training in distributed machine learning (ML) systems, focusing on enhancing both efficiency and practicality. As distributed ML systems become integral for complex ML tasks, ensuring resilience against Byzantine failures—where workers may contribute incorrect updates due to malice or error—gains paramount importance. Our first contribution is the introduction of the Centered Trimmed Meta Aggregator (CTMA), an efficient meta-aggregator that upgrades baseline aggregators to optimal performance levels, while requiring low computational demands. Additionally, we propose harnessing a recently developed gradient estimation technique based on a double-momentum strategy within the Byzantine context. Our paper highlights its theoretical and practical advantages for Byzantine-robust training, especially in simplifying the tuning process and reducing the reliance on numerous hyperparameters. The effectiveness of this technique is supported by theoretical insights within the stochastic convex optimization (SCO) framework and corroborated by empirical evidence.
\end{abstract}

\section{Introduction}
In modern machine learning (ML), the paradigm of large-scale distributed training systems
has emerged as a cornerstone for advancing complex ML tasks. Distributed ML approaches enable to significantly accelerate the training process; thus facilitating the practical use of larger, more sophisticated models~\citep{zhao2023survey}.
However, as these systems grow in scale and complexity, they become increasingly susceptible to a range of faults and errors.
Moreover, distributed ML also propels collaborative learning across decentralized data sources~ \cite{bonawitz2019towards}, which often differ in distribution, quality, and volume~\cite{bonawitz2019towards}. For example, data from different geographic locations, devices, or organizations can exhibit considerable variability. This poses a critical challenge: ensuring the training process is resilient to faults and errors in such distributed and heterogeneous environments.  Fault-tolerant training becomes imperative to maintain the integrity, accuracy, and reliability of the learned models, especially when the stakes involve critical decision-making based on ML predictions.

The Byzantine model~\cite{lamport2019Byzantine,guerraoui2023byzantine} provides a robust framework for devising and analyzing fault-tolerant training in distributed ML, due to its capability of capturing both random and adversarial failures. It addresses the challenge posed by Byzantine workers, who may introduce incorrect updates due to unpredictable or malicious behavior, thus risking the training process. Owing to its generality, the Byzantine setting has been widely adopted as a framework for devising and analyzing fault-tolerant ML approaches \citep{blanchard2017machine, yin2018Byzantine}.

Addressing Byzantine failures in distributed machine learning (ML) requires robust aggregation rules to mitigate malicious or faulty updates from workers. Various robust aggregation rules have been developed in the past years, including Coordinate-wise Trimmed Mean (CWTM)  \citep{yin2018Byzantine}, Krum \citep{blanchard2017machine}, Geometric Median (GM) \citep{chen2017distributed}, CWMed \citep{yin2018Byzantine}, and Minimum Diameter Averaging (MDA) \citep{guerraoui2018hidden}. Each such aggregator brings unique qualities that may better address different types of Byzantine failures. 
The diversity in aggregator designs underscores the necessity of having a versatile toolbox, as no single method can efficiently handle all possible Byzantine scenarios. 

The efficiency and robustness of these aggregators can be systematically characterized, as shown in  \citet{karimireddy2020Byzantine, karimireddy2021learning, allouah2023fixing,farhadkhani2022Byzantine}. These studies not only outline the worst-case strengths of various aggregators but also set benchmarks for optimal efficiency, providing a comprehensive analysis of existing algorithms. Furthermore, they introduce the concept of \textbf{Meta-Aggregators}—advanced mechanisms that enhance the robustness of base aggregators by combining or modifying their outputs. Meta-aggregators are devised to bolster the system’s resilience, leveraging the strengths of underlying methods to achieve greater fault tolerance. The first meta-aggregator that aimed to enhance baseline robust aggregators was Bucketing~\citep{karimireddy2020Byzantine}. However, this approach relies heavily on randomization, which can result in unstable performance in practice,  as shown by \citet{allouah2023fixing}. To address this, Nearest Neighbor Mixing (NNM)~\citep{allouah2023fixing} introduced a more effective approach in practice that enhanced the performance of any baseline robust aggregator. Nevertheless, the application of NNM comes with high computational costs, creating a trade-off between enhanced robustness and efficiency. Notably, the computational demand for NNM is $m$ times higher compared to the standard ``average" aggregator, where $m$ is the number of workers. This significant increase in computational demand limits its practicality in high-dimensional scenarios and in situations where the number of workers is large.

The challenge of applying robust aggregators to standard stochastic gradients in synchronous robust training has been a notable concern, as highlighted in the work of \citet{karimireddy2021learning}. The core issue stems from Byzantine workers exploiting the variance in gradient estimations to introduce biases, potentially derailing the learning process. To counteract this, \citet{karimireddy2021learning} have proposed to utilize momentum estimates rather than direct gradient estimates. This approach, underpinned by rigorous theoretical analysis, advocates for each honest machine to maintain a momentum that averages past stochastic gradients across approximately $\sqrt{T}$ iterations, where $T$ is the total number of updates. Such a strategy achieves a significant reduction in variance by a factor of $\sqrt{T}$, effectively minimizing the ``noise blanket" that Byzantine workers could exploit to inject harmful bias into the system \cite{karimireddy2021learning}.

On the downside, implementing this momentum-based method within Byzantine environments unveils a critical limitation: the necessity for precise tuning of both the momentum parameter and the learning rate. This fine-tuning process requires a deep understanding of several factors, including the variance of noise, the fraction of Byzantine workers present, and the performance (or strength) of the chosen robust aggregator. Such prerequisites for effective application pose practical challenges, as they demand prior comprehensive knowledge about the system's specific characteristics and adversarial conditions. This complexity restricts the straightforward applicability of momentum-based approaches in diverse, real-world settings.

In this work, we develop new tools for Byzantine-robust training. Our focus is on enhancing efficiency and practicality. Our contributions: 
\begin{itemize}
\item \textbf{Efficient Meta-Aggregator (CTMA):} We present  ``Centered Trimmed Meta Aggregator (CTMA)", a novel meta-aggregator that upgrades baseline aggregators to optimal performance levels. CTMA offers stable performance in practice compared to Bucketing while significantly reducing computational demands typically associated with NNM, aligning its efficiency with that of the conventional ``average" aggregator. This development makes CTMA suitable for widespread deployment in large-scale scenarios, balancing practical stability with low computational demands and facilitating the implementation of robust, fault-tolerant distributed machine learning systems.
\item \textbf{Advanced Gradient Estimation Technique}: Building on the insight that standard stochastic gradient methods are vulnerable to Byzantine disruptions, we suggest incorporating
a recent double-momentum technique \citep{levy2023mu} to enhance resilience during the training process. 
 Unlike traditional momentum-based approaches that necessitate complex parameter tuning, our approach simplifies the process by allowing the learning rate to be determined solely by the objective's smoothness parameter $L$. This approach not only streamlines the setup but also achieves a more substantial error reduction, further mitigating the impact of Byzantine workers and enhancing the practicality of synchronous robust training in distributed ML systems. We establish theoretical guarantees for this approach within the Stochastic-Convex-Optimization (SCO)--
 a fundamental framework for the design and analysis of machine learning algorithms~\citep{shalev2009stochastic}.
\end{itemize}
We also demonstrate the usefulness of our approach in practice.
\paragraph{Related Work.} 
Addressing the Byzantine problem in distributed ML has evolved significantly, with historical information playing a pivotal role. Early work by \citet{alistarh2018Byzantine} utilizes the entire historical information to mitigate Byzantine faults effectively. Subsequently, momentum-based approaches emerged as a popular and straightforward solution to enhance robustness by leveraging a subset of \(\sqrt{T}\) past gradients \citep{allen2020Byzantine,farhadkhani2022Byzantine,el2021distributed,karimireddy2020Byzantine,karimireddy2021learning}. \citet{karimireddy2021learning} further highlights the necessity of historical information, showing that methods without it fail when addressing strongly convex objectives. This underscores the essential role of using historical data to ensure the robustness and reliability of learning algorithms in Byzantine settings.

In this context, traditional models often assume data homogeneity. However, data heterogeneity—where data across workers vary in distribution—is more common and presents significant challenges. The studies by \citet{allouah2023fixing} and \citet{karimireddy2020Byzantine} are particularly relevant to our work, demonstrating the effectiveness of incorporating momentum and robust aggregators in heterogeneous settings. These studies also pioneered the exploration of meta-aggregators, which are specifically designed to enhance the performance of baseline robust aggregators. \citet{karimireddy2020Byzantine} introduce Bucketing, a meta-aggregator that averages each worker's output with those of its nearest neighbors, identified through a random permutation, and then sends these processed outputs to a robust aggregator. Building on this concept, \citet{allouah2023fixing} propose Nearest Neighbor Mixing (NNM), a meta-aggregator that averages outputs based on their nearest neighbors defined by Euclidean norms before sending them to a robust aggregator.

\section{Setting}
We consider a Heterogeneous distributed setting comprising $m$ workers, where a subset $\GGG$ of them are honest. Conversely to the Homogeneous case, we assume each honest worker $i\in\GGG$ may independently draw i.i.d.~samples from a distribution $\D^{(i)}$, and that the different data distributions may vary between workers. We consider stochastic optimization problems, focusing on smooth convex objective functions given by $f_i:\K \rightarrow \real$ for each honest worker $i$. Our goal is to minimize the joint objective of all honest workers by $f:\K\mapsto\reals$:
\begin{equation*}
    f(\bx):=\frac{1}{|\GGG|}\sum_{i\in\GGG} f_i(\bx):=\frac{1}{|\GGG|}\sum_{i\in\GGG}\E_{\bz^{(i)}\sim\D^{(i)}} f_i(\bx;\bz^{(i)}),
\end{equation*}
where $\K\subseteq\real^d$ is a compact convex set. Thus, the objective is an average of $|\GGG|$ functions $\{f_i:\K\mapsto\reals\}_{i\in\GGG}$, and  each such $f_i(\cdot)$ can be written as an expectation over losses $f_i(\cdot,\bz^{(i)})$ where $\bz^{(i)}$ is drawn from some distribution $\D^{(i)}$ which is unknown to the learner. For ease of notation, in what follows, we will not explicitly denote $\E_{\bz^{(i)}\sim\D^{(i)}}[\cdot]$  but rather use $\E[\cdot]$ to denote the expectation w.r.t.~all randomization.

We consider iterative first-order optimization algorithms that leverage the gradients of  $f(\cdot)$ to generate a series of query points. The final point in this series, denoted as $\bx_T$, serves as an estimation of the optimal solution. The performance of these algorithms is evaluated based on the expected excess loss, defined as:
\begin{equation*}
    \textrm{ExcessLoss}:=\E[f(\bx_T) - f(\bx^*)]~,
\end{equation*}
where $\bx^* \in \arg\min_{\bx\in\K} f(\bx)$. The goal of the workers is to collaboratively ensure a small excess loss w.r.t.~common objective $f(\cdot)$.

We consider a distributed environment, where each worker $i$ has the ability to compute a stochastic gradient oracle $\bg\in\real^d$, given by:
    $\bg^{(i)}:=\nabla f_i(\bx; \bz^{(i)})$~, 
for some $\bz^{(i)}\sim \D^{(i)}$, implying unbiasedness, i.e., $\E[\bg^{(i)}|\bx]=\nabla f_i(\bx)$.  We focus on a centralized framework characterized by a central Parameter Server ($\mathcal{PS}$). In this setup, the $\mathcal{PS}$ communicates with all the workers in the network. Specifically,  our main focus is on synchronous systems, where the $\mathcal{PS}$ waits for outputs from all workers before updating its global vector and then distributes this updated vector to the workers for their next step; in the spirit of minibatch-SGD~\cite{dekel2012optimal}.

Lastly, we consider scenarios where up to a fraction of \( \delta < 1/2 \) are Byzantine workers exhibiting problematic behaviors, such as sending arbitrary or malicious information during the training process. These "Byzantine" workers may even collude to disrupt the training process. The identities of these workers are unknown. We denote \( \GGG \subseteq [m] \) as the set of honest workers, with a size of \( |\GGG| \geq (1 - \delta)m \). Conversely, \( \B \subseteq [m] \) represents the set of Byzantine workers, whose size is \( |\B| \leq \delta m \).

\paragraph{Notation.} Throughout, $\lVert\cdot\rVert$ represents the $L_2$-norm, and for any natural $N$, we define $\left[N\right]=\left\{1,\ldots,N\right\}$. We use a compressed sum notation, where $\alpha_{1:t}= \sum_{k=1}^{t}{\alpha_k}$. In addition, for every $\bx\in\mathbb{R}^d$, we denote the orthogonal projection of $\bx$ onto a set $\K$ by \(\Pi_\K(\bx)= \arg\min_{\by \in \K} \|\by - \bx\|^2\).

\paragraph{Assumptions.} $\forall i\in\GGG$, we use the following assumptions: \\
\textbf{Bounded Diameter}: we assume there exists $D>0$ such,
\begin{equation}
\label{eq:bounded_diameter}
    \max_{\bx,\by\in\K}\|\bx-\by\|\leq D~.
\end{equation}
\textbf{Bounded Variance}:~there exists $\sigma>0$ such that,
\begin{align} \label{eq:bounded-variance}
	\E\|\nabla f_i(\bx;\bz^{(i)}) - \nabla f_i(\bx)\|^2\leq \sigma^2~;\quad \forall \bx\in\K~.
\end{align}
\textbf{Heterogeneity}:~
there exists $\xi>0$ such that for any $\bx\in \K$,
\begin{align} 
  \frac{1}{|\GGG|}\sum_{i\in\GGG}\| \nabla f_i(\bx)-\nabla f(\bx)\|^2 \leq \xi^2 ~.
\label{eq:hetro}
\end{align} 
\textbf{Expectation over Smooth Functions}:~
we assume that $f_i(\cdot)$ is an expectation of smooth functions, i.e.~ $\forall \bx,\by\in\K~, \bz^{(i)}\in \textbf{Support}\{\D^{(i)}\} $ there exist $L>0$ such that,
\begin{align} \label{eq:Main}
\|\nabla f_i(\bx;\bz^{(i)}) - \nabla f_i(\by;\bz^{(i)})\| \leq L\|\bx-\by\|~,
\end{align} 
The above assumption also implies that the expected loss $f_i(\cdot)$ and the averaged loss $f(\cdot)$ are $L$ smooth. \\
\textbf{Bounded Smoothness Variance} \citep{levy2023mu}:~
note that the assumption that we make in Eq.~\eqref{eq:Main} implies that, $\forall \bx,\by\in\K$ there exists $\sigmal^2 \in[0,L^2]$ such,
\begin{align} 
\begin{split}
    \E\left\|(\nabla f_i(\bx;\bz)-\nabla f_i(\bx)) - (\nabla f_i(\by;\bz)-\nabla f_i(\by))\right\|^2
    \leq \sigmal^2 \|\bx-\by\|^2 ~. 
\end{split}
\label{eq:sigmal}
\end{align}  In Appendix \ref{sec:sigmal} we show that   Eq.~\eqref{eq:Main} implies Eq.~\eqref{eq:sigmal}.

\section{Robust Aggregation and Meta-Aggregators}
In Byzantine-free distributed environments, the server typically employs gradient averaging as an aggregation rule~\cite{dekel2012optimal}. This approach is particularly effective due to its ability to execute gradient computations in parallel and efficiently reduce the variance of noisy gradients. However, in Byzantine settings, such averaging may completely fail the training process, even in the face of a single Byzantine worker. Thus, it is crucial for the server to employ a robust aggregation rule; which leads to the following update,
\begin{align*}
    \bw_{t+1} = \Pi_{\K}\left(\bw_{t} -\eta \A \left(\bg_t^{(1)}, \ldots, \bg_t^{(m)}\right)\right)~.
\end{align*}
Here \(\A\) represents  a robust aggregation function applied to the gradients \(\bg_t^{(1)}, \ldots, \bg_t^{(m)}\) from the $m$ workers at time \(t\), and $\eta$ is the learning rate.

In the context of robust methods, several works \citep{allouah2023fixing, karimireddy2021learning,karimireddy2020Byzantine, farhadkhani2022Byzantine} have developed and characterized a range of robust aggregation rules. The key property of such aggregators is their ability to limit the variance between the output of a robust aggregator and the average output of honest workers.  In our discussion, we adopt the robust definition of \citet{karimireddy2021learning} with a few modifications. This refined definition, detailed in Definition \ref{def1}, guarantees that the expected deviation of an aggregation rule from the average output of honest workers is limited to \(c_\delta \geq 0\) times the variance of these workers' outputs.

\begin{definition} \textnormal{\( (c_\delta, \delta) \)-\textbf{robust}}.
\label{def1}
Assume we have $m$ random vectors $\bx_1, \ldots, \bx_m \in \real^d$. Also assume we have an "honest" subset \(\GGG \subseteq [m]\), implying
$\{\bx_i\}_{i\in \GGG}$ are independent of each other.  Finally, assume that there exists $\delta\in [0,1/2)$ such that $|\GGG| \geq (1-\delta)m$. Moreover, assume that for any $i\in\GGG$ there exist $\rho_i\geq 0$ such that, 
\begin{gather*}
    \E \|{\bx_i - \Bar{\bx}_\GGG}\|^2 \leq \rho_i^2~,
\end{gather*}
where \ $\bar{\bx}_\GGG := \frac{1}{|\GGG|} \sum_{i \in \GGG} {\bx}_i$. Then an aggregation rule $\A$ is called \( (c_\delta, \delta) \)-robust is for any such $m$ random vectors it outputs $\hat{\bx}\gets \A(\bx_1, \ldots, \bx_m)$ such that,
\begin{gather*}
    \E\|{\hat{\bx} - \bar{\bx}_\GGG} \|^2 \leq c_\delta \rho^2~,
\end{gather*}
where $\rho^2:=\frac{1}{|\GGG|}\sum_{i\in\GGG}\rho_i^2$, and the expectation w.r.t.~$\{\bx_i\}_{i=1}^m$ and (possible) randomization in  $\A$.
\end{definition}

 Note that \cite{allouah2023fixing} has provided a different definition, capturing a rich family of robust aggregators (see Def.2 in \cite{allouah2023fixing}). In Lemma \ref{lem:kappa-robust}, we show that any aggregator satisfying their definition will also satisfy our definition above.
    Thus, upper bounds of \(c_\delta\) for various aggregation rules are based on the findings in \cite{allouah2023fixing} and are elaborated in Table \ref{tab:agg}. Furthermore, \cite{allouah2023fixing} has established that a lower bound for $c_\delta$ is $\frac{\delta}{1-2\delta}$ and thus, the optimal rate of $c_\delta < O(\delta)$ is achieved when $\delta < 1/3$. Additionally, this optimal rate is effectively achieved by the CWTM aggregator, as demonstrated in Table \ref{tab:agg}. The performance of the other aggregators in Table \ref{tab:agg} is sub-optimal.
\begin{table}
    \centering
    \begin{tabular}{|l|c|c|c|c|c|c|}
        \hline
        Aggregation & CWTM & Krum& GM*  & CWMed & Avg. \\
        \hline
        \({c_\delta}\) & \( \frac{\delta}{{1-2\delta}} \left( 1 + \frac{\delta}{{1-2\delta}}\right) \) & \( 1 + \frac{\delta}{{1-2\delta}} \) & \( \left(1 + \frac{\delta}{{1-2\delta}}\right)^2 \) & \( \left(1 + \frac{\delta}{{1-2\delta}}\right)^2 \) & - \\
        \hline
        Computational cost & \( O(d m \log m) \) & \( O(d m^2) \) & \( O(dm + d\epsilon^{-2}) \) & \( O(dm\log{m}) \) & \( O(dm) \) \\
        \hline
    \end{tabular}
    \caption{Summary of aggregation rules (CWTM \citep{yin2018Byzantine}, Krum \citep{blanchard2017machine}, GM* ($\epsilon$-approximate GM) \citep{chen2017distributed, acharya2022robust}, CWMed \citep{yin2018Byzantine}) with respective \(c_\delta\) values and computational costs.}
    \label{tab:agg}
\end{table}

\subsection{Centered Trimmed Meta Aggregator (CTMA)}
\label{sec:ccra}

To address the sub-optimal performance of \((c_\delta,\delta)\)-robust aggregators, previous works have proposed "meta-aggregators," which enhance the performance of baseline sub-optimal aggregators. Concretely, such meta-aggregators employ a baseline $(c_\delta,\delta)$-robust aggregators in a black box manner to yield an improved robust aggregator with a factor of $O(\delta(1+c_\delta))$ instead of $c_\delta$. Thus, even if $c_\delta = O(1)$ (and therefore sub-optimal); then after meta-aggregation we receive an order optimal factor of $O(\delta)$. This, in turn, enables the crafting of a wide class of optimal aggregators, which can be highly beneficial in practice.

Currently, two approaches exist for meta-aggregation: Bucketing \citep{karimireddy2020Byzantine} Nearest Neighbor Mixing (NNM) \citep{allouah2023fixing}. Bucketing was the first proposed meta-aggregator, but it proved to be less stable in practice due to its heavy reliance on randomization \citep{allouah2023fixing}. To address this, NNM was introduced as a more practical alternative; however, its computational complexity is notably high, requiring $O(d m^2)$ computations in the worst case. This is substantially higher compared to the computational cost of many prevalent robust aggregators, as shown in Table \ref{tab:agg}, and limits its practicality in large-scale (large \(d\)) and massively parallel (large \(m\)) scenarios. 

To mend the above issue, we propose an alternative meta-aggregator in Algorithm ~\ref{alg:ctma}, named Centered Trimmed Meta Aggregator (CTMA). CTMA achieves guarantees comparable to NNM and Bucketing while requiring a computational cost of $O(dm + m\log m)$,  which is similar (up to a logarithmic factor) to that of the standard average aggregator. Later in this paper, we will also demonstrate its practical effectiveness.

\begin{algorithm}
\caption{Centered Trimmed Meta Aggregator (CTMA)}
\label{alg:ctma}
\begin{algorithmic}[t]
\State \textbf{Input:}  Sequence of vectors \( \bx_1, \bx_2, \ldots, \bx_m \), a $(c_\delta,\delta)$-robust aggregator \( \A \), $\delta\in[0,1/2)$.
\State \textbf{Initailize:} $\bx_0\gets \A(\bx_1, \ldots, \bx_m)$. 
\State Sort the sequence \(\{ \|\bx_i - \bx_0\| \}_{i=1}^m\) in non-decreasing order.
\State  Define \( S \gets \) set of indices corresponding to the first $(1-\delta)m$ elements in the sorted sequence.
\State Compute \( \hat{\bx} = \frac{1}{|S|} \sum_{i \in S} \bx_i \).
\State \Return \( \hat{\bx} \)
\end{algorithmic}
\end{algorithm}

As outlined in Algorithm \ref{alg:ctma}, the CTMA process begins by sorting the sequence \(\{ \|\bx_i - \bx_0\| \}_{i=1}^m\) in non-decreasing order. It then computes the average of the first \( (1-\delta)m \) vectors in this ordered sequence, selecting those based on their minimal deviation from the robust aggregator \(\bx_0\). The idea behind this approach is to redefine the role of the robust aggregator \(\bx_0\). Instead of using \(\bx_0\) directly, it serves as an anchor point representing a robust reference. By averaging its nearest neighbors, this method leverages the strength of the collective while preserving robustness, leading to enhanced performance.

It is immediate to see that Algorithm~\ref{alg:ctma} requires $O(dm+m\log m)$ computations, in addition to computing $\bx_0$. Furthermore, as a post-processing approach, CTMA can be effectively combined with either NNM or Bucketing with only an additional logarithmic factor at most, offering a versatile solution in optimizing robust aggregation processes. Note that the CTMA can be applied to a wide range of optimization problems, including non-convex problems.

\begin{lemma}
\label{lem:CTMA}
Under the assumptions outlined in Definition \ref{def1}, if CTMA receives a $(c_\delta, \delta)$-robust aggregator, $\A$; then the output of CTMA, $\hat{\bx}$, is  $(16\delta(1+c_\delta), \delta)$-robust.
\begin{proof}
For simplicity lets assume w.l.o.g.~ that $|\GGG| = (1-\delta) m= |S|$. At Remark \ref{remark:ctma}, we describe how to extend to the general case where $|\GGG| \geq (1-\delta) m = |S|$.

We denote $\by_i := \bx_i -\bx_0$.
\begin{align*}
    &\hat{\bx} - \bar{\bx}_\GGG = \frac{1}{|S|} \sum_{i \in S}\bx_i - \bar{\bx}_\GGG  \\
   &= \bx_0  - \bar{\bx}_\GGG + \frac{1}{|\GGG|} \sum_{i \in S} (\bx_i - \bx_0) \quad (|S|=|\G|) \\ &= -\frac{1}{|\GGG|} \sum_{i \in \GGG} \by_i + \frac{1}{|\GGG|} \sum_{i \in S} \by_i     \\
    &= -\frac{1}{|\GGG|} \sum_{i \in \GGG} \by_i + \frac{1}{|\GGG|} \sum_{i \in \GGG} \by_i
     - \frac{1}{|\GGG|} \sum_{i \in \GGG\backslash S} \by_i + \frac{1}{|\GGG|} \sum_{i \in S\backslash \GGG} \by_i \\
    &= - \frac{1}{|\GGG|} \sum_{i \in \GGG\backslash S} (\bx_i - \bx_0) + \frac{1}{|\GGG|} \sum_{i \in S\backslash \GGG} (\bx_i - \bx_0).
\end{align*}
Taking the squared norm of both sides and applying the Jensen’s inequality, we obtain:
\begin{align*}
    &\|\hat{\bx} - \bar{\bx}_\GGG\|^2 = \left\| - \frac{1}{|\GGG|} \sum_{i \in \GGG\backslash S} (\bx_i - \bx_0) + \frac{1}{|\GGG|} \sum_{i \in S\backslash \GGG} (\bx_i - \bx_0) \right\|^2 \\
    &\leq \frac{2|S\backslash \GGG|}{|\GGG|^2} \sum_{i \in S\backslash \GGG} \|\bx_i - \bx_0\|^2 + \frac{2|\GGG\backslash S|}{|\GGG|^2} \sum_{i \in \GGG\backslash S} \|\bx_i - \bx_0\|^2 ~.
\end{align*}
Note that $|\GGG\backslash S| = |\GGG\cup S|-|S| \leq m - |\GGG| = |\B|$ and in a similar way $|S\backslash \GGG| \leq |\B|$. Therefore,  
\begin{align*}
\|\hat{\bx} - \bar{\bx}_\GGG\|^2 & \leq \frac{4\delta}{|\GGG|} \sum_{i \in S\backslash \GGG} \|\bx_i - \bx_0\|^2 + \frac{4\delta}{|\GGG|} \sum_{i \in \GGG\backslash S} \|\bx_i - \bx_0\|^2~.
\end{align*}
Next, we show that there exists an injective function $\Phi:S\backslash\GGG\to \GGG\cap \bar{S}$. To achieve this, we first denote \( |\B \cap S| = q \). It is important to note that \( \B \cap S = S\setminus\GGG \), which leads us to derive the bound for \( |\GGG \cap \bar{S}| \):
\begin{align*}
    |\bar{S} \cap \GGG| &= |\bar{S}| - |\bar{S} \cap \B| = |\B| - |\B\setminus (S \cap \B)| \\ & = |\B| - (|\B| - |S \cap \B|) = q~.
\end{align*}
Consequently, it implies that \( |\GGG \cap \bar{S}| = |S\setminus\GGG|=q \). Therefore, we can assert the existence of an injective function \( \Phi: S\setminus\GGG \to \GGG \cap \bar{S} \). Further, by the definition of the set $S$, \(\forall i\in \GGG \cap \bar{S} \), \(\forall j\in S\setminus\GGG \) we have that $\|\bx_i-\bx_0\|\geq\|\bx_j-\bx_0\|$. Thus,
\begin{align*}
&\|\hat{\bx} - \bar{\bx}_\GGG\|^2  \leq \frac{4\delta}{|\GGG|}(  \sum_{i \in S\backslash \GGG} \|\bx_{\Phi(i)} - \bx_0\|^2 +  \sum_{i \in \GGG\backslash S} \|\bx_i - \bx_0\|^2) \\ & \leq \frac{8\delta}{|\GGG|} \sum_{i \in \GGG} \|\bx_i - \bx_0\|^2 ~.
\end{align*}
Taking the expectations of both sides gives us the following:
\begin{align*}
&\E\|\hat{\bx} - \bar{\bx}_\GGG\|^2  \leq \frac{8\delta}{|\GGG|} \sum_{i \in \GGG} \E\|\bx_i - \bx_0\|^2 \\ & \leq \frac{16\delta}{|\GGG|} \sum_{i \in \GGG} \E\|\bx_i - \Bar{\bx}_\GGG\|^2 + \frac{16\delta}{|\GGG|} \sum_{i \in \GGG} \E\|\Bar{\bx}_\GGG - \bx_0\|^2 \\  & \leq {16\delta}\rho^2 + 16\delta c_\delta \rho^2  =  16\delta(1+c_\delta) \rho^2~,
\end{align*}
where the second inequality uses $\|\ba+\bb\|^2\leq 2\|\ba\|^2+2\|\bb\|^2$, which holds $\forall \ba,\bb\in\real^d$.  The last inequality stems from the assumption in Def.~\ref{def1}. 
\begin{remark}
\label{remark:ctma}
    CTMA can also receive an upper bound estimate of \(\delta\) and adjust the size of \( S \) to \((1-\delta)m\). In our analysis, we assume, without loss of generality, that this estimate matches the true size of \(\B\). Should this not be the case, our proof can alternatively rely on a subset of \(\GGG\) of size \((1-\delta)m\) without affecting the validity of our conclusions.
\end{remark}
\end{proof}
\end{lemma}

\section{Synchronous Robust Training}
Previous works have shown that applying a robust-aggregator to the standard stochastic gradients is bound to fail~\citep{karimireddy2020Byzantine}. This is due to the fact that Byzantine workers may make use of the ``noise blanket" of standard gradient estimators in order to inject harmful bias, which may totally fail the learning process.  Nevertheless, \citet{karimireddy2020Byzantine,allouah2023fixing} have devised an elegant and simple solution to this issue in the form of aggregating momentum estimates rather than gradient estimates. Concretely, the theoretical analysis in this work suggests that each (honest) worker maintains a momentum that effectively averages the past $\approx \sqrt{T}$ stochastic gradients, which leads to a variance reduction by a factor of $\sqrt{T}$. This shrinking of the ``noise blanket"  limits the harmful bias that may be injected by Byzantine workers. Here we take a similar approach. Nevertheless, instead of using the momentum, we make use of a recent approach called $\mu^2$-SGD \citep{levy2023mu}, which enables a more aggressive error reduction of factor $t$ compared to standard gradient estimates. One benefit of this approach over utilizing momentum, is that it will allow us to pick the momentum parameter in a completely parameter-free manner, and pick the learning rate only based on the smoothness  $L$; Conversely, previous methods that rely on momentum require picking both the momentum and the learning rate based on $L$, on $c_\delta$ and on the variance of stochastic gradients~\citep{karimireddy2020Byzantine,allouah2023fixing}.
Next, we overview this approach and describe and analyze it in the context of Synchronous Robust training.

\subsection{$\mu^2$-SGD}
The $\mu^2$-SGD is a variant of standard SGD with several modifications. Its update rule is of the following form, $\bw_1=\bx_1\in\K$, and $\forall t>1$,
\begin{gather*}
    \label{eq:mu2sgd}
    \bw_{t+1}=\Pi_{\K}\left({\bw_{t}- \eta\alpha_{t} \bd_{t}}\right), \
    \bx_{t+1}=\frac{1}{\alpha_{1:t+1}} \sum_{k\in[t+1]} \alpha_k \bw_k~.
\end{gather*}
Here, $\{\alpha_t>0\}_t$ are importance weights that may unequally emphasize different update steps; concretely we will employ
$\alpha_t\propto t$, which puts more emphasis on the more recent updates. Moreover, the $\{\bx_t\}_t$'s are a sequence of weighted averages of the iterates $\{\bw_t\}_t$, and $\bd_t$ is an estimate for the gradient at the average point, i.e.~of $\nabla f(\bx_t)$. This is different than standard SGD, which employs estimates for the gradients at the iterates, i.e.~of $\nabla f(\bw_t)$.  This approach is related to a technique called Anytime-GD~\citep{cutkosky2019anytime}, which is strongly-connected to the notions of momentum and acceleration~\citep{cutkosky2019anytime, kavis2019unixgrad}. 

While in the natural stochastic version of Anytime-GD, one would use the estimate $\nabla f(\bx_t;\bz_t)$;  the $\mu^2$-SGD approach suggests to employ a variance reduction mechanism to yield a \emph{corrected momentum} estimate $\bd_t$. This is done as follows: $\bd_1: =\nabla f(\bx_1;\bz_1)$, and $\forall t>2$,
\begin{equation*}
\bd_{t}=\nabla f(\bx_{t};\bz_{t}) + (1-\beta_{t})(\bd_{t-1} - \nabla f(\bx_{t-1};\bz_{t}))~,
\end{equation*}
where $\beta_t\in[0,1]$ are called \emph{corrected momentum} weights.
It can be shown by induction that \(\E[\bd_{t}] = \E[\nabla f(x_t)]\); however, in general, \(\E[\bd_{t} \mid x_t] \neq \nabla f(x_t)\) (in contrast to standard SGD estimators). Nevertheless, \citet{levy2023mu} demonstrated that by choosing \emph{corrected momentum} weights \(\beta_t := 1/t\), the above estimate enjoys an aggressive error reduction. Specifically, \(\E\|\varepsilon_t\|^2 := \E\|\bd_t - \nabla f(\bx_t)\|^2 \leq O(\tsigma^2/t)\) at step \(t\), where \(\tsigma^2 \leq O(\sigma^2 + D^2 \sigma_L^2)\). Implying that the variance decreases with $t$, contrasting with standard SGD where the variance  $\E\|\varepsilon^{\textnormal{SGD}}_t\|^2 := \E\|\bg_t - \nabla f(\bx_t)\|^2$ remains uniformly bounded.

\subsection{Synchronous Robust \(\mu^2\)-SGD}
We consider a synchronous training approach in the spirit of Minibatch-SGD~\cite{dekel2012optimal}, where the $\mathcal{PS}$ updates its global vector by aggregating information from all workers. This aggregation occurs only after receiving outputs from all workers. Afterward, the $\mathcal{PS}$ sends the updated global vector back to each worker, allowing them to compute their next output.

\label{sec:sync}

\begin{algorithm}[t]
\caption{Synchronous Robust \(\mu^2\)-SGD}\label{alg:sync}
\begin{algorithmic}
    \State \textbf{Input:} learning rate \(\eta_t > 0\), starting point \(\bx_{1} \in \K\), steps number \(T\), importance weights \(\{\alpha_t\}_t\), corrected momentum weights \(\{\beta_t\}_t\), \((c_\delta, \delta)\)-robust aggregation function \(\mathcal{A}\).
    \State \textbf{Initialize:} set \(\bw_1 = \bx_1\), draw \(\bz_1^{(i)} \sim \mathcal{D}^{(i)}\), set \(\bd_1^{(i)}=\nabla f(\bx_1;\bz_1^{(i)})\), \(\forall i \in [m]\).
    \For{\(t = 1, \ldots, T\)} \quad \(\triangleright\) \textit{server update}
        \begin{gather*}
            \bw_{t+1} = \Pi_{\mathcal{K}}\left(\bw_{t} - \eta\alpha_{t} \mathcal{A}\left(\bd_{t}^{(1)}, \ldots, \bd_{t}^{(m)}\right)\right) \\
            \bx_{t+1} = \frac{1}{\alpha_{1:t+1}} \sum_{k \in [t+1]} \alpha_k \bw_k
        \end{gather*}
        \For{\(i = 1, \ldots, m\)} \quad \(\triangleright\)  \textit{worker update}
            \State draw \(\bz_{t+1}^{(i)} \sim \mathcal{D}\), compute \(\bg_{t+1}^{(i)} = \nabla f(\bx_{t+1};\bz_{t+1}^{(i)})\), \(\Tilde{\bg}_{t}^{(i)} = \nabla f(\bx_{t};\bz_{t+1}^{(i)})\),
            \State and update:
            \begin{equation*}
                \bd_{t+1}^{(i)} = \bg_{t+1}^{(i)} + (1 - \beta_{t+1})(\bd_{t}^{(i)} - \Tilde{\bg}_{t}^{(i)}) 
            \end{equation*}
        \EndFor
    \EndFor
    \State \textbf{Output:} \(\bx_T\)
\end{algorithmic}
\end{algorithm}

Algorithm \ref{alg:sync} describes the adapted  \(\mu^2\)-SGD algorithm for the synchronous Byzantine setting. In this approach, each worker computes its own corrected momentum $\bd_t^{(i)}$ independently using its own data samples:
\begin{equation*}
\label{eq:storm-worker}
    \bd^{(i)}_{t}=\nabla f_i(\bx_{t};\bz^{(i)}_{t})+(1-\beta_{t})(\bd^{(i)}_{t-1}-\nabla f_i(\bx_{t-1};\bz^{(i)}_{t}))~.
\end{equation*}
Once the outputs from all workers are received, the $\mathcal{PS}$ then proceeds to update the new query point   \(\bx_{t+1}\):
\begin{align*}
\label{eq:anytime-sync}
    &\bw_{t+1}=\Pi_{\K}\left({\bw_{t}- \eta\alpha_{t} \A\left(\bd_{t}^{(1)} \ldots, \bd_{t}^{(m)}\right) }\right),\\
    &\bx_{t+1}=\frac{1}{\alpha_{1:t+1}} \sum_{k\in[t+1]} \alpha_k \bw_k~,
\end{align*}
where we are assumed that $\A\left(\bd_{t}^{(1)}, \ldots, \bd_{t}^{(m)}\right)$  is a \((c_\delta, \delta)\)-robust aggregator (Def.~\ref{def1}).

The motivation for utilizing the \(\mu^2\)-SGD algorithm in the Byzantine setting lies in its dual mechanism, which incorporates weighted averaging over the entire history of both query points and gradients. This allows the \(\mu^2\)-SGD algorithm to perform careful gradient steps with small changes, thereby enhancing the stability of the gradient descent process. This stability leads to aggressive stochastic variance reduction, which can be valuable for identifying Byzantine noise (or malicious workers) over time.

\begin{theorem}
\label{thm:Main}
For each worker \(i \in \GGG\), assume a convex function \(f_i:\K\mapsto\real\)  on a convex set \(\K\) with bounded diameter \(D\). Under the assumptions given in Equations~\eqref{eq:bounded-variance}, \eqref{eq:Main}, and \eqref{eq:sigmal}, the implementation of Algorithm~\ref{alg:sync} with parameters \(\{\alpha_t = t\}_t\) and \(\{\beta_t = 1/t\}_t\) guarantees the following for every \(t \in [T]\) and each honest worker \(i \in \GGG\):
\begin{align*}
    &\E\left\|\varepsilon^{(i)}_t\right\|^2 = \E\left\|\bd^{(i)}_t -  \nabla f_i(\bx_t)\right\|^2 \leq  \tsigma^2 / t ~,\\ &\E\left\|\frac{1}{|\GGG|}\sum_{i\in\GGG}\varepsilon^{(i)}_{t}\right\|^2  \leq \tsigma^2 / t|\GGG|~,
    \end{align*}
where \(\varepsilon^{(i)}_{t} = \bd^{(i)}_{t}-\nabla f_i(\bx^{(i)}_{t})\) and \(\tsigma^2:=2\sigma^2 + 8D^2 \sigma_L^2\).
\end{theorem}
Theorem \ref{thm:Main} shows that this method is consistent with the guarantees of the non-distributed \(\mu^2\)-SGD \citep{levy2023mu}, preserving the same stochastic error \(\E \left\|\varepsilon^{(i)}_t\right\| := \E\left\|\bd^{(i)}_t -  \nabla f_i(\bx_t)\right\|\) for each honest worker. Furthermore, the collective error across honest workers results in a variance reduction proportional to the number of workers, aligning with the principles of mini-batch SGD \citep{dekel2012optimal}.

As previously discussed, the double mechanism of the \(\mu^2\)-SGD method offers significant variance reduction over the standard mini-batch SGD by minimizing stochastic variance at each step \(t\). Consequently, integrating the \(\mu^2\)-SGD algorithm with a \((c_\delta, \delta)\)-robust aggregator ensures a consistent reduction in variance between the filter's output and the actual gradient,  as demonstrated in Lemma \ref{lem:syncFilter}.

\begin{lemma}
\label{lem:syncFilter}
Under the Byzantine assumption where \(\delta < 1/2\), let \(\A\) be a \((c_\delta, \delta)\)-robust aggregation rule, and let \(f: \K \mapsto \real\) be a convex function, where \(\K\) is a convex set with bounded diameter \(D\). Presuming that the assumptions in Equations~\eqref{eq:bounded-variance}, \eqref{eq:hetro}, \eqref{eq:Main}, and \eqref{eq:sigmal} hold, invoking Algorithm~\ref{alg:sync} with \(\{\alpha_t = t\}_t\) and \(\{\beta_t = 1/t\}_t\) ensures for any \(t \in [T]\):
\begin{gather*}
    \E\left\|\hat{\bd}_t - \nabla f(\bx_t) \right\|^2 \leq   \frac{4\tsigma^2}{tm} + \frac{12c_\delta\tsigma^2}{t} + 6c_\delta\xi^2~,
\end{gather*}
where $\hat{\bd}_t=\A\left(\bd^{(1)}_t, \ldots, \bd^{(m)}_t\right)$ and \(\tsigma^2:=2\sigma^2 + 8D^2 \sigma_L^2\).

\begin{proof}
We start by determining the appropriate value for $\rho$ as outlined in Definition \ref{def1}, where $\bar{\bd}_\GGG := \frac{1}{|\GGG|}\sum_{i\in\GGG} \bd^{(i)}_t$:
\begin{align*}
     &\E \left\|{\bd^{(i)}_t  - \bar{\bd}_\GGG}\right\|^2  
    \\ & \leq 3\E \left\|{\bd^{(i)}_t -  \nabla f_i(\bx_t)}\right\|^2 + 3\E \left\|{ \nabla f(\bx_t) -\bar{\bd}_\GGG }\right\|^2 + 3\E \left\|{ \nabla f_i(\bx_t) -\nabla f(\bx_t) }\right\|^2 \\ &\leq \frac{3\tsigma^2}{t} + \frac{3\tsigma^2}{t|\GGG|} + 3\E \left\|{ \nabla f_i(\bx_t) -\nabla f(\bx_t) }\right\|^2 \\ & \leq \frac{6\tsigma^2}{t} + 3\E \left\|{ \nabla f_i(\bx_t) -\nabla f(\bx_t) }\right\|^2~.
\end{align*}
The first and inequality uses $\|\ba+\bb+\bc\|^2\leq 3\|\ba\|^2+3\|\bb\|^2+3\|\bc\|^2$, which holds $\forall \ba,\bb,\bc\in\real^d$. The second inequality follows Theorem \ref{thm:Main} and the assumption in Eq. \eqref{eq:hetro}. The third inequality is due to the fact that $|\GGG|\geq 1$. Thus,
\begin{align*}
     &\frac{1}{|\GGG|} \sum_{i\in\GGG} \E \left\|{\bd^{(i)}_t  - \bar{\bd}_\GGG}\right\|^2   \\ & \leq \frac{6\tsigma^2}{t} + \frac{3}{|\GGG|} \sum_{i\in\GGG}\E \left\|{ \nabla f_i(\bx_t) -\nabla f(\bx_t) }\right\|^2 \\ & \leq \frac{6\tsigma^2}{t} + 3\xi^2~.
\end{align*}
Accordingly, we set $\rho^2:=\frac{6\tsigma^2}{t} + 3\xi^2$. Following Definition \ref{def1}, we have,
\begin{align*}
   &\E\left\|\hat{\bd}_t - \nabla f(\bx_t) \right\|^2 \leq  
    2\E\left\|\hat{\bd}_t - \bar{\bd}_\GGG \right\|^2 + 2\E\left\|\bar{\bd}_\GGG - \nabla f(\bx_t) \right\|^2 \\
    &\leq 2c_\delta\left(\frac{6\tsigma^2}{t} + 3\xi^2\right) +  2\E \left\|\frac{1}{|\GGG|}\sum_{i\in\GGG}\varepsilon_t^{(i)}\right\|^2  \\
    &\leq \frac{12c_\delta\tsigma^2}{t} +  \frac{2\tsigma^2}{t|\GGG|} + 6c_\delta\xi^2 \\& \leq \frac{12c_\delta\tsigma^2}{t} +  \frac{4\tsigma^2}{tm} + 6c_\delta\xi^2~,
\end{align*}
where the first inequality uses $\|\ba+\bb\|^2\leq 2\|\ba\|^2+2\|\bb\|^2$, which holds $\forall \ba,\bb\in\real^d$. The third follows Theorem \ref{thm:Main}. The last inequality utilizes the fact that $|\GGG|\geq(1-\delta)m\geq m/2$ since $\delta < 1/2$.
\end{proof}

\end{lemma}
This lemma breaks down the error between the filter's output and the actual gradient into three components: the collective stochastic error across all honest workers, \(c_\delta\) times the stochastic error for an individual honest worker, and \(c_\delta\) times the heterogeneity variance. In the absence of Byzantine workers (\(\delta = 0\)), \(c_\delta\) can be reduced to zero if \(c_\delta \leq O(\delta)\). This condition can be met by using an appropriate robust aggregator (see Table \ref{tab:agg}) or by applying a meta-aggregator like CTMA with a suitable robust aggregator. This aligns with the ideal $\mu^2$-SGD analysis in a Byzantine-free setting \citep{levy2023mu}.

\begin{theorem}[Synchronous Byzantine $\mu^2$-SGD]\label{thm:muSGD}
Assume $f$ is convex. 
Also, let us make the same assumptions as in Thm.~\ref{thm:Main} and Lemma \ref{lem:syncFilter}, and let us denote $G^*:=\| \nabla f(\bx^*)\|$, where $ \bx^* \in \arg\min_{\bx\in\K} f(\bx)$. Then invoking Algorithm~\ref{alg:sync} with $\{\alpha_t = t\}_t$ and $\{\beta_t=1/t\}_t$, and using a learning rate $\eta\leq \frac{1}{4LT}$ guarantees,
\begin{equation*}
    \E[\Delta_T] \leq O\left(\frac{G^*D+LD^2}{ T}  +  \frac{D\tsigma}{\sqrt{T}}\sqrt{c_\delta + \frac{1}{m}} +\sqrt{c_\delta}D\xi\right)
\end{equation*} 
where $\Delta_T:=f(\bx_T)-f(\bx^*)$ and \(\tsigma^2:=2\sigma^2 + 8D^2 \sigma_L^2\).
\end{theorem}
\begin{remark}
    Theorem \ref{thm:muSGD} provides optimal excess loss bound when $c_\delta\leq O(\delta)$ \citep{karimireddy2021learning}. Additionally, this bound is consistent with the optimal analysis in a Byzantine-free setting ($\delta=0$) \citep{levy2023mu, dekel2012optimal}.
\end{remark}

\begin{remark}
    Theorem \ref{thm:muSGD} indicates that the learning rate solely relies on the smoothness parameter $L$. This is in contrast to the standard momentum approach (e.g.,~\citet{karimireddy2020Byzantine,allouah2023fixing}), which necessitates prior knowledge of $L,\sigma$, and $c_\delta$ to set the learning rate and momentum parameter.
\end{remark}

\begin{remark}
In Appendix \ref{par:lr}, we demonstrate that our approach can derive a wide range of learning rates \(\eta \in [\eta_{min}, \eta_{max}]\) that ensure optimal convergence. Unlike the standard momentum method, which requires a low-range of learning rates such that \(\eta_{max}/\eta_{min} = O(1)\) \citep{karimireddy2020Byzantine} to achieve an order optimal bound, our method allows for a broader range of learning rates with \(\eta_{max}/\eta_{min} = O(\sqrt{T})\). This provides significantly greater flexibility in fine-tuning the learning rate parameter.
\end{remark}

\section{Experiments}

In this section, we present the numerical performance of the CTMA and \(\mu^2\)-SGD algorithms. We assess their robustness against \textit{sign-flipping} and \textit{label-flipping} attacks \citep{allen2020Byzantine}, as well as state-of-the-art (SOTA) attacks, including the \textit{empire} attack \citep{xie2020fall} and the \textit{little} attack \citep{baruch2019little}. Additionally, we perform a comparative analysis of CTMA alongside two existing methods: Bucketing and NNM. Furthermore, we explore the effectiveness of employing a double meta-aggregator approach, where NNM serves as the initial meta-aggregator before passing through a robust aggregation phase, followed by the application of CTMA.

Our experiments also measure the performance of the $\mu^2$-SGD algorithm compared to the standard momentum \citep{karimireddy2021learning}, extending the evaluation to examine their resilience across a wide range of learning rates.

We conducted our experiments in a homogeneous setting. Specifically, we utilized the MNIST \citep{lecun2010mnist} dataset, which contains 28x28 pixel grayscale images of handwritten digits, and the CIFAR-10 \citep{krizhevsky2014cifar} dataset, which includes 32x32 color images spanning 10 classes. The detailed experimental setup and the complete results are provided in Appendix \ref{app:exp}. Our results demonstrate consistent outcomes across both the CIFAR-10 and MNIST datasets.

For the code, please visit our GitHub repository\footnote{\url{https://github.com/dahan198/synchronous-fault-tolerant-ml}}.

\paragraph{CTMA Versus Existing Meta Aggregators.} Our results of CTMA show that it often performed as well as, and in some cases even better than, existing meta-aggregators such as Bucketing and NNM for both the momentum and \(\mu^2\)-SGD algorithms. As illustrated in Figures \ref{fig:ctma-vs-meta} and \ref{fig:sota-ctma}, and demonstrated in Appendix \ref{app:exp}, CTMA enhances the performance of baseline robust-aggregators, exhibiting exceptional robustness and accuracy under challenging conditions. However, Figure \ref{fig:ctma-fails} indicates that CTMA may struggle with heavy low-variance attacks, such as \textit{little} and \textit{empire}, which fall under the stochastic "noise blanket" and complicate the identification of adversarial workers, leading to performance degradation. Despite this, its computational efficiency and effectiveness against other realistic attacks, such as \textit{sign-flipping} and \textit{label-flipping}, as well as under weaker low-variance attacks (see Figures \ref{fig:ctma-low-mnist} and \ref{fig:ctma-low-cifar}), make CTMA a valuable tool in the field of meta-aggregation.

Additionally, integrating NNM with CTMA can even further improve performance. For instance, Figure \ref{fig:ctma-vs-meta} shows that in the case of the \textit{sign-flipping} attack with Robust Aggregation for Federated Learning (RFA) \citep{pillutla2022robust} (an approximation of the GM aggregator), the combination of NNM and CTMA significantly improved both robustness and accuracy for both momentum and $\mu^2$-SGD.

\begin{figure*}[ht]
\centering
\subfigure[Performance comparison of CTMA with existing meta-aggregators under \textit{sign-flipping} and \textit{label-flipping} attacks.]{
    \includegraphics[width=0.45\textwidth]{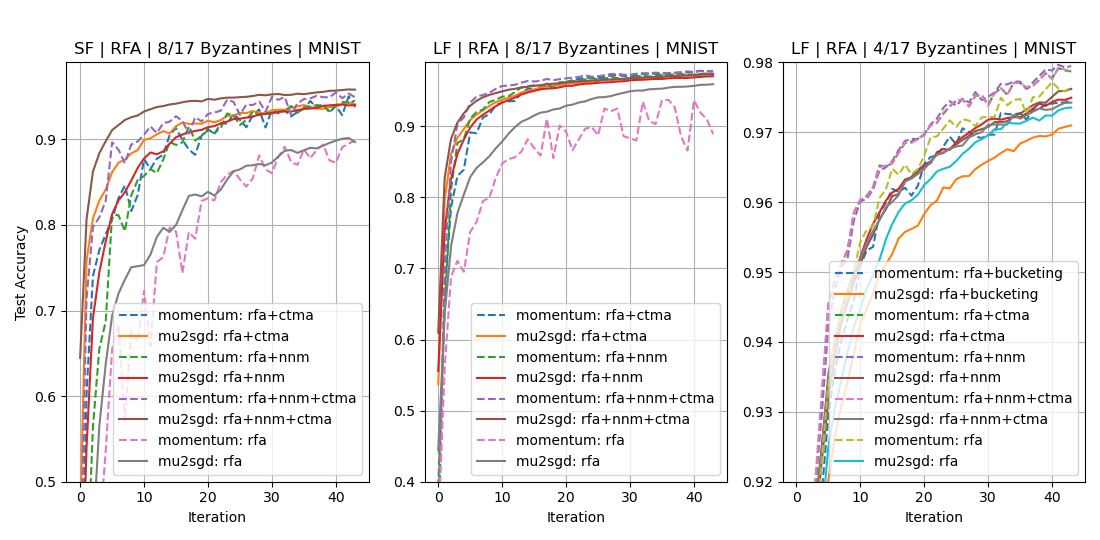}
    \label{fig:ctma-vs-meta}
}
\hfill
\subfigure[Performance comparison of CTMA with existing meta-aggregators under \textit{empire} and \textit{little} attacks.]{
    \includegraphics[width=0.45\textwidth]{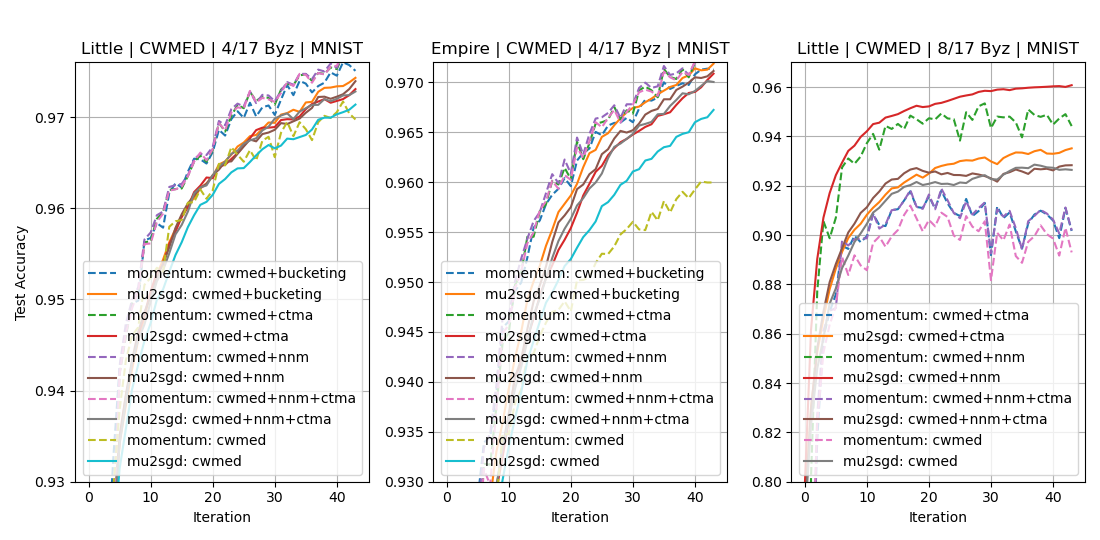}
    \label{fig:sota-ctma}
}
\caption{Performance comparison of CTMA with existing meta-aggregators (Conf. 1 in Table \ref{tab:configurations}).}
\label{fig:side-by-side}
\end{figure*}

\begin{figure*}[ht]
\centering
\subfigure[Performance comparison of $\mu^2$-SGD with standard momentum under \textit{empire} and \textit{little} attacks (Conf. 1 in Table \ref{tab:configurations}).]{
    \includegraphics[width=0.45\textwidth]{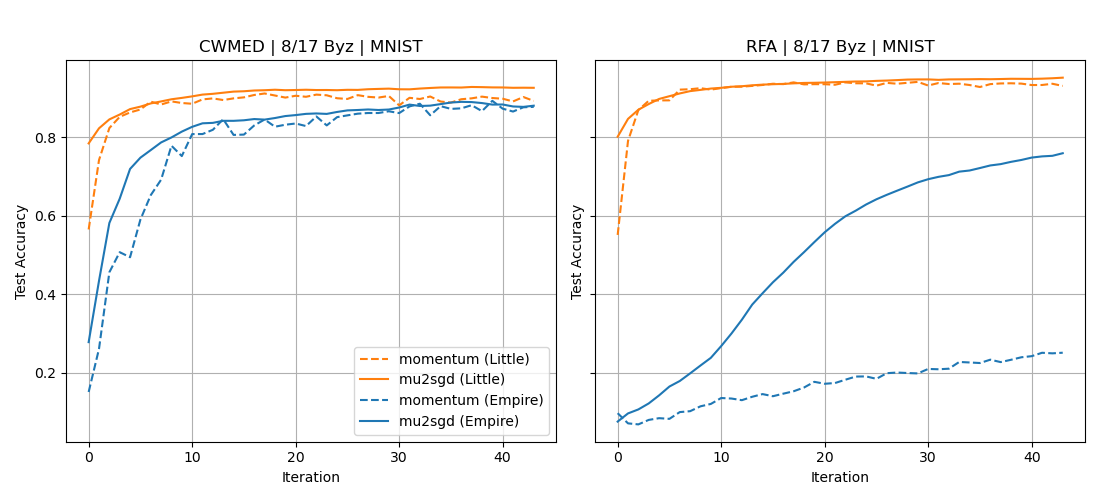}
    \label{fig:sota-mu2sgd}
}
\hfill
\subfigure[Performance comparison of $\mu^2$-SGD with standard momentum for a wide range of learning rates, under \textit{sign-flipping} and \textit{label-flipping} attacks (Conf. 4 in Table \ref{tab:configurations}).]{
    \includegraphics[width=0.45\textwidth]{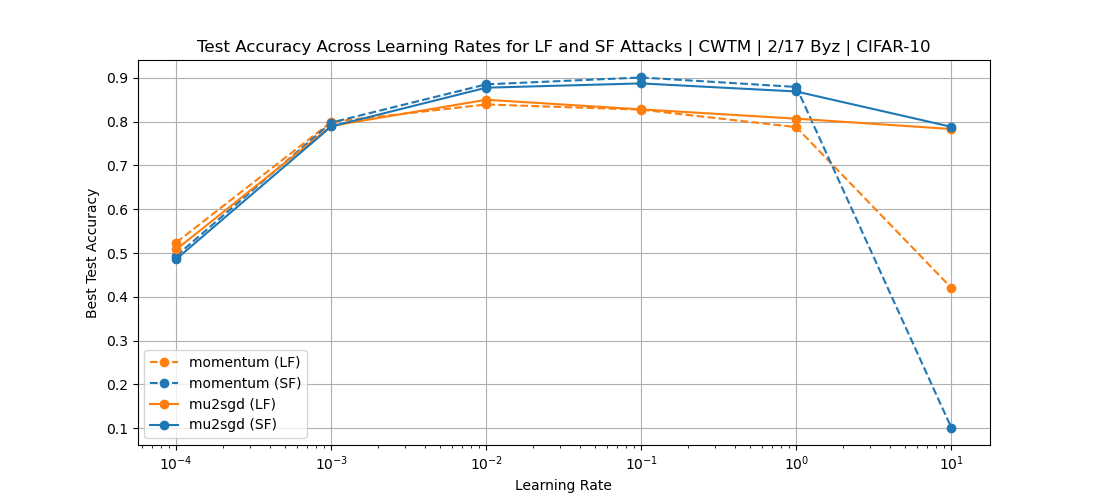}
    \label{fig:lr}
}
\caption{Performance comparison of $\mu^2$-SGD with standard momentum.}
\label{fig:side-by-side2}
\end{figure*}

\paragraph{$\mu^2$-SGD Versus Momentum.}

We compared the performance of the \(\mu^2\)-SGD algorithm against the standard momentum across a wide range of learning rates and various attacks. The results in Figures \ref{fig:sota-mu2sgd} and \ref{fig:ctma-vs-meta} show that \(\mu^2\)-SGD outperformed momentum under severe attacks when nearly half of the workers were Byzantine. Figure \ref{fig:sota-mu2sgd} highlights \(\mu^2\)-SGD's strong resilience against heavy SOTA attacks like \textit{little} and \textit{empire}. In these cases, \(\mu^2\)-SGD exhibited superior convergence speed and stability during training iterations by considering the entire history of gradients and query points, which aligns well with our theory. However, in cases of weaker attacks, as shown in Figure \ref{fig:ctma-vs-meta} for the \textit{label-flipping} attack with 4 out of 17 Byzantine workers, or when a meta-aggregator is used to enhance performance, as depicted in Figure \ref{fig:sota-ctma}, the high stability of \(\mu^2\)-SGD is not necessary. In fact, in these scenarios, the noisier momentum achieved faster convergence. This is because \(\mu^2\)-SGD, with its double momentum mechanism and parameters that account for the entire history, takes very small and cautious steps, potentially slowing its convergence compared to a single momentum mechanism with a larger parameter.

Furthermore, Figure \ref{fig:lr} shows that \(\mu^2\)-SGD exhibited greater stability across a wide range of learning rates for \textit{label-flipping} and \textit{sign-flipping} attacks, whereas momentum tended to be more sensitive to this parameter. This provides a broader range of options for hyper-parameter optimization of \(\mu^2\)-SGD.

\section*{Conclusions and Future Work}
We introduce two complementary techniques to enhance the efficiency and practicality of Byzantine-robust training. First, we present a novel meta-aggregator that substantially improves efficiency compared to previous methods. Second, we incorporate a recent novel gradient estimation technique into Byzantine-robust training. We substantiate the theoretical benefits of our new approaches and corroborate them in practice, demonstrating an advantage over momentum, the standard workhorse of Byzantine-robust training.

Future directions: \textbf{(i)} extending our investigation to the challenging setting of decentralized Byzantine-robust training \cite{he2022Byzantine}; 
\textbf{(ii)} understand whether we can further develop estimates that incorporate (stochastic) second order information in the spirit of~\citet{antonakopoulos2022extra}; in order to yield an even more aggressive variance reduction, which will give rise to easier mitigation of faulty updates.

\section*{Acknowledgement}
This research was partially supported by Israel PBC-VATAT, by the Technion Artificial Intelligent Hub (Tech.AI) and by the Israel Science Foundation (grant No. 447/20).

\bibliographystyle{plainnat}
\bibliography{bib}

\newpage
\appendix

\section{Bounded Smoothness Variance Assumption}
\label{sec:sigmal}
We show that Eq.~\eqref{eq:Main} implies that Eq.~\eqref{eq:sigmal} holds for some $\sigmal^2 \in[0,L^2]$.
\begin{align*}
\E\|(\nabla f_i(\bx;\bz)-\nabla f_i(\bx)) - (\nabla f_i(\by;\bz)-\nabla f_i(\by))\|^2 &= 
\E\|\nabla f_i(\bx;\bz)-\nabla f_i(\by;\bz)\|^2 - \|\nabla f_i(\bx)-\nabla f_i(\by))\|^2 \\
&\leq L^2 \|\bx-\by\|^2  ~.
\end{align*} 
Here, we also used  $\E(\nabla f_i(\bx;\bz)-\nabla f_i(\by;\bz))=(\nabla f_i(\bx)-\nabla f_i(\by))$, and followed Eq.~\eqref{eq:Main}. Therefore, we establish that  $\sigmal^2\in[0,L^2]$.

\section{Synchronous $\mu^2$-SGD Analysis}
\subsection{Proof of Thm.~\ref{thm:Main}}
\label{sec:proofThmEpsilon}
\begin{proof}[Proof of Thm.~\ref{thm:Main}] 

We follow similar steps as in the proof of Theorem 4.1 of \citet{levy2023mu}.  First, we consider the definition of $\bx_t$, which leads us to the following relationship:
        \begin{gather*}
            \alpha_{1:t-1}(\bx_t-\bx_{t-1})=\alpha_t(\bw_t-\bx_t)~,
        \end{gather*}
which further yields to,
\begin{equation*}
    \|{\bx_t-\bx_{t-1}}\|=(\alpha_t/\alpha_{1:t-1})\|{\bw_t-\bx_t}\|~.
\end{equation*}
Given that $\bw_t$ belongs to $\K$ by its definition, and $\bx_t$ is a weighted average of $\{\bw_t\}_t$, the convexity of $\K$ ensures that $\bx_t \in \K$. Furthermore, $\|\bw_t - \bx_t\| \leq D$ in accordance with the assumption in Eq. \eqref{eq:bounded_diameter}. Assigning $\alpha_t = t$, leads us to establish the ratio $\frac{\alpha_t}{\alpha_{1:t-1}} = \frac{2}{t-1}$. Consequently, for any $t\geq 1$,
\begin{equation}
\label{eq:anytime_diameter}
    \|{\bx_t-\bx_{t-1}}\| \leq \frac{2}{t-1} D~.
\end{equation}
We proceed to analyze the recursive dynamics of $\varepsilon_t^{(i)}$ for each $i\in {\GGG}$. Based on the definitions of $\bd_t^{(i)}$ and $\varepsilon_t^{(i)}$, we can present the recursive relationship in the following way:
\begin{align*}
\varepsilon_t^{(i)}=\beta_t(\bg_{t}^{(i)}-\nabla f_i(\bx_t))+(1-\beta_t)Z^{(i)}_t+(1-\beta_t)\varepsilon^{(i)}_{t-1}~,
\end{align*}
where $Z^{(i)}_t := (\bg_{t}^{(i)}-\nabla f_i(\bx_t))-(\Tilde{\bg}_{t-1}^{(i)}-\nabla f_i(\bx_{t-1}))$. Upon choosing $\beta_t = \frac{1}{t}$, we can reformulate the equation as follows: 
\begin{equation*}
    t\varepsilon_t^{(i)}=(\bg_{t}^{(i)}-\nabla f_i(\bx_t))+(t-1)Z^{(i)}_t+(t-1)\varepsilon^{(i)}_{t-1}=\MMM_t^{(i)}+(t-1)\varepsilon^{(i)}_{t-1}~,
\end{equation*}
where $\MMM_t^{(i)}:=(\bg_{t}^{(i)}-\nabla f_i(\bx_t))+(t-1)Z^{(i)}_t$. Unrolling this recursion yields an explicit expression for any $t \in [T]$:
\begin{equation*}
t\varepsilon_t^{(i)}=\sum_{\tau\in[t]}\MMM_\tau^{(i)}~.
\end{equation*}
Following this, we derive an upper bound for the expected square norm of $\MMM_\tau^{(i)}$ as follows:
\begin{align*}
    \E\|{\MMM_\tau^{(i)}}\|^2 &= \E\|{(\bg_{\tau}^{(i)}-\nabla f_i(\bx_\tau))+(\tau-1)Z^{(i)}_\tau}\|^2 \\
    &\leq 2\E\|{\bg_{\tau}^{(i)}-\nabla f_i(\bx_\tau)}\|^2 + 2\E\|{(\tau-1)Z^{(i)}_\tau}\|^2 \\
    &= 2\E\|{\bg_{\tau}^{(i)}-\nabla f_i(\bx_\tau)}\|^2 + 2(\tau-1)^2\E\|{(\bg_{\tau}^{(i)}-\nabla f_i(\bx_\tau))-(\Tilde{\bg}_{\tau-1}^{(i)}-\nabla f_i(\bx_{\tau-1}))}\|^2 \\
    &\leq 2\sigma^2 + 2\sigma_L^2(\tau-1)^2 \E\|\bx_\tau-\bx_{\tau-1}\|^2 \\
    &\leq 2\sigma^2 + 8D^2\sigma_L^2 = \tsigma^2~,
\end{align*}
where the first inequality uses $\|\ba+\bb\|^2\leq 2\|\ba\|^2+2\|\bb\|^2$, which holds $\forall \ba,\bb\in\real^d$. The second inequality employs the assumptions outlined in Equations \eqref{eq:bounded-variance} and \eqref{eq:sigmal}. The third inequality follows Eq. \eqref{eq:anytime_diameter}. 

Note that for each \( i \in \GGG \),  the sequence $\{\MMM^{(i)}_t\}_t$ constitutes a martingale difference sequence relative to a natural filtration $\{\F^{(i)}_t\}_t$. Furthermore, the sequence $\{\MMM^{(i)}_{t}\}_{t,i}$ forms \( |\GGG| \) independent martingale difference $\{\MMM^{(i)}_t\}_t$ sequences relative to a natural filtration $\F_t$. 

\begin{lemma}\label{lem:doubleSumMart}
Consider $\{M^{(i)}_{t}\}_{t,i}$, a collection of $m$ martingale difference sequences $\{M^{(i)}_t\}_t$ with a respect to a natural filtration $\F_t$. These sequences are independent for each $i\in[m]$ with a respect to natural filtration $\{\F^{(i)}_t\}_t$. Then, for any $t\geq 1$, the following holds:
\begin{align*}
\E \left\|\sum_{i\in[m]}\sum_{\tau\in[t]}M_{\tau}^{(i)}\right\|^2 =
\sum_{i\in[m]}\sum_{\tau\in[t]} \E\left\|M^{(i)}_\tau\right\|^2~.
\end{align*}
\end{lemma}

Leveraging Lemma \ref{lem:doubleSumMart}, we get the following relationship:
\begin{gather*}            
\E\left\|{t\sum_{i\in\GGG}\varepsilon_t^{(i)}} \right\|^2  = \E\left\|{\sum_{i\in\GGG}\sum_{\tau\in[t]}\MMM_\tau^{(i)}}\right\|^2 
= \sum_{i\in\GGG}\sum_{\tau\in[t]}\E\left\|{\MMM_\tau^{(i)}}\right\|^2
{\leq} \tsigma^2t|\GGG|~.
\end{gather*} 
In a similar way, for the individual error, we have:
\begin{gather*}
    \E\left\|t{\varepsilon_t^{(i)}} \right\|^2 = \E\left\|{\sum_{\tau\in[t]}\MMM_\tau^{(i)}}\right\|^2 =  \sum_{\tau\in[t]}\E\left\|{\MMM_\tau^{(i)}}\right\|^2 \leq \tsigma^2t~.
\end{gather*}
\end{proof}

\subsubsection{Proof of Lemma ~\ref{lem:doubleSumMart}}
\begin{proof}[Proof of Lemma ~\ref{lem:doubleSumMart}]
\begin{lemma}[Borrowed from Lemma B.1 in \citet{levy2023mu}]
\label{lem:SumMart}
Let  $\{ M_t\}_t$ be a martingale difference sequence with respect to a filtration $\{\F_t\}_t$, then the following holds for any $t$,
\begin{align*}
\E \left\|\sum_{\tau\in[t]} M_\tau \right\|^2 
&=
\sum_{\tau\in[t]} \E\left\| M_\tau\right\|^2 ~.
\end{align*}
\end{lemma}

For any $\tau\in[t]$, we have that,
\begin{align}
\E \left\|\sum_{i\in[m]}M_{\tau}^{(i)}\right\|^2 
&= \sum_{i\in[m]}\E\left\|M_{\tau}^{(i)}\right\|^2 + \sum_{i \neq j; \ i,j\in[m]} \E\langle M_{\tau}^{(i)}, M_{\tau}^{(j)} \rangle \nonumber
\\&=
\sum_{i\in[m]}\E\left\|M_{\tau}^{(i)}\right\|^2 + \sum_{i \neq j; \ i,j\in[m]} \E\left[\underbrace{\E\left[\langle M_{\tau}^{(i)}, M_{\tau}^{(j)} \rangle | \F_{\tau-1}\right]}_{=0}\right] \nonumber 
\\
&=
\sum_{i\in[m]}\E\left\|M_{\tau}^{(i)}\right\|^2 ~,\label{eq:workers-sum}
\end{align}
where the second equality is obtained by applying the law of total expectation. For the third equality, we rely on the independence of \(M_{\tau}^{(i)}\) and \(M_{\tau}^{(j)}\) for each \(i, j \in [m]\), \(i \neq j\) and the fact that \(\{M^{(i)}_{\tau}\}_\tau\) constitutes a martingale difference sequence, which implies \(\E\left[ M_{\tau}^{(i)} \middle| \F_{\tau-1}\right] = 0\). 

Therefore,
\begin{align*}
\E \left\|\sum_{i\in[m]}\sum_{\tau\in[t]}M_{\tau}^{(i)}\right\|^2 
= \E \left\|\sum_{\tau\in[t]}\sum_{i\in[m]}M_{\tau}^{(i)}\right\|^2 = \sum_{\tau\in[t]}\E \left\|\sum_{i\in[m]}M_{\tau}^{(i)}\right\|^2 = \sum_{\tau\in[t]}\sum_{i\in[m]}\E \left\|M_{\tau}^{(i)}\right\|^2~,
\end{align*}
where the second equality is a result of applying Lemma \ref{lem:SumMart}, which is appropriate since the sequence $\left\{\sum_{i\in[m]}M_{t}^{(i)}\right\}_t$ forms a martingale difference sequence. The third equality follows the result in Eq. \eqref{eq:workers-sum}.
\end{proof}

\subsection{Proof of Thm.~\ref{thm:muSGD}}
\begin{proof}[Proof of Thm.~\ref{thm:muSGD}] 
\label{sec:muSGDProof}
We begin by revisiting the AnyTime guarantee as outlined in \citet{cutkosky2019anytime} and adapt the standard regret analysis of the update rule, as detailed in \citet{hazan2016introduction}, to our context.
\begin{theorem}[Rephrased from Theorem 1 in \citet{cutkosky2019anytime}]
\label{theorem1}
    Let $f:\K\rightarrow\real$ be a convex function with a minimum $\bx^*\in\arg\min_{\bw\in\K}f(\bw)$. Also let $\{\alpha_t\geq 0\}_t$, and $\{\bw_t\in\K\}_t$, $\{\bx_t\in\K\}_t$, such that $\{\bx_t\}_t$ is an $\{\alpha_t\}_t$ weighted averaged of $\{\bw_t\}_t$, i.e. such that $\bx_1=\bw_1$, and for any $t\geq 1$,
    \begin{equation*}
        \bx_{t+1}=\frac{1}{\alpha_{1:t+1}}\sum_{\tau\in[t+1]}{\alpha_\tau\bw_\tau}~.
    \end{equation*}
    Then the following holds for any $t\geq 1$:
    \begin{equation*}
        \alpha_{1:t}(f(\bx_t)-f(\bx^*))\leq\sum_{\tau\in[t]}\alpha_\tau\nabla f(\bx_\tau)(\bw_\tau-\bx^*)~.
    \end{equation*}
\end{theorem}

\begin{lemma} \label{lem:RegretBound}
Let $f:\K\rightarrow\real$ be a convex function with a minimum $\bx^*\in\arg\min_{\bw\in\K}f(\bw)$, and assume that the assumption in Eq. \eqref{eq:bounded_diameter} holds. Also let $\{\alpha_t\geq 0\}_t$, and $\{\bw_t\in\K\}_t$. Then, for any $t\geq 1$, an arbitrary vector $\hat{\bd}_t\in\real^d$, and the update rule:
\begin{gather*}
    \bw_{t+1}=\Pi_{\K}\left({\bw_{t}- \eta\alpha_{t}\hat{\bd}_t}\right)~,
\end{gather*}
we have, 
\begin{align*}
    \sum_{\tau=1}^t \alpha_\tau \langle \hat{\bd}_\tau, \bw_{\tau+1}-\bx^*\rangle \leq 
    \frac{D^2}{2\eta} -\frac{1}{2\eta}\sum_{\tau=1}^t \|\bw_{\tau}-\bw_{\tau+1}\|^2~.
\end{align*}
\end{lemma}

\begin{lemma}
\label{lemma2}
let $f:\K \rightarrow \real$ be an $L$-smooth and convex function, and let $\bx^*\in\arg\min_{\bx\in\K}f(\bx) $, then for any $\bx\in\real^d$ we have,
\begin{equation*}
    \|{\nabla f(\bx)}-\nabla f(\bx^*)\|^2 \leq 2L(f(\bx)-f(\bx^*))~.
\end{equation*}
\end{lemma}

Next, for every iteration $t\leq T$, we define:
\begin{gather*}
    \hat{\bd}_t:=\A\left(\bd^{(1)}_t, \ldots, \bd^{(m)}_t\right)~,
    \\
    \varepsilon_t:=\hat{\bd}_t-\nabla f(\bx_t)~.
\end{gather*}
Thus, combininh Theorem~\ref{theorem1} with Lemma~\ref{lem:RegretBound}, we have that,
\begin{align}
\label{eq:regret-main}
    \alpha_{1:t}(f(\bx_t)-f(\bx^*)) &\leq \sum_{\tau\in[t]} \alpha_\tau\langle \nabla f(\bx_\tau), \bw_\tau-\bx^*\rangle \nonumber \\ 
    &= \sum_{\tau\in[t]}\alpha_\tau\langle \hat{\bd}_\tau,\bw_{\tau+1}-\bx^*\rangle + \sum_{\tau\in[t]}\alpha_\tau\langle \hat{\bd}_\tau,\bw_\tau-\bw_{\tau+1}\rangle - \sum_{\tau\in[t]}\alpha_\tau \langle\varepsilon_\tau,\bw_\tau-\bx^*\rangle \nonumber \\
    &\leq \frac{D^2}{2\eta} - \frac{1}{2\eta} \sum_{\tau\in[t]} \|{\bw_\tau-\bw_{\tau+1}}\|^2  + \sum_{\tau\in[t]}\alpha_\tau\langle \hat{\bd}_\tau,\bw_\tau-\bw_{\tau+1}\rangle  - \sum_{\tau\in[t]}\alpha_\tau\langle\varepsilon_\tau,\bw_\tau-\bx^*\rangle \nonumber \\
    &= \frac{D^2}{2\eta} - \frac{1}{2\eta} \sum_{\tau\in[t]} \|{\bw_\tau-\bw_{\tau+1}}\|^2 +  \sum_{\tau\in[t]}\alpha_\tau\langle \nabla f(\bx_\tau),\bw_\tau-\bw_{\tau+1}\rangle  -\sum_{\tau\in[t]}\alpha_\tau \langle\varepsilon_\tau,\bw_{\tau+1}-\bx^*\rangle \nonumber \\  &\leq \frac{D^2}{2\eta} \underbrace{- \frac{1}{2\eta} \sum_{\tau\in[t]} \|{\bw_\tau-\bw_{\tau+1}}\|^2 +  \sum_{\tau\in[t]}\alpha_\tau\langle \nabla f(\bx_\tau),\bw_\tau-\bw_{\tau+1}\rangle  }_{\textnormal{(A)}}+D\sum_{\tau\in[t]}\alpha_\tau\|{\varepsilon_\tau}\| ~,
\end{align}
where the first inequality is derived from the Anytime guarantee, as outlined in Theorem \ref{theorem1}. The second inequality follows Lemma \ref{lem:RegretBound}.  The third inequality is a result of applying the Cauchy-Schwarz inequality and the assumption in Eq. \eqref{eq:bounded_diameter}.
\begin{align*}
    \textnormal{(A)} &:=  - \frac{1}{2\eta} \sum_{\tau\in[t]} \|{\bw_\tau-\bw_{\tau+1}}\|^2 +  \sum_{\tau\in[t]}\alpha_\tau\langle \nabla f(\bx_\tau),\bw_\tau-\bw_{\tau+1}\rangle
    \\ & =- \frac{1}{2\eta} \sum_{\tau\in[t]} \|{\bw_\tau-\bw_{\tau+1}}\|^2 +  \sum_{\tau\in[t]}\alpha_\tau\langle \nabla f(\bx_\tau)-\nabla f(\bx^*),\bw_\tau-\bw_{\tau+1}\rangle + \sum_{\tau\in[t]}\alpha_\tau  \langle \nabla f(\bx^*),\bw_\tau-\bw_{\tau+1}\rangle \\ 
    & \leq \frac{\eta}{2} \sum_{\tau\in[t]} \alpha_\tau^2\|{\nabla f(\bx_\tau)-\nabla f(\bx^*)}\|^2 + \sum_{\tau\in[t]}\alpha_\tau\langle \nabla f(\bx^*),\bw_\tau-\bw_{\tau+1}\rangle\\ 
    & \leq 2\eta L \sum_{\tau\in[t]} \alpha_{1:\tau}\Delta_\tau + \sum_{\tau\in[t]}(\alpha_\tau-\alpha_{\tau-1})\langle \nabla f(\bx^*),\bw_\tau\rangle - \alpha_t \langle \nabla f(\bx^*),\bw_{t+1}\rangle \\ 
    & = 2\eta L \sum_{\tau\in[t]} \alpha_{1:\tau}\Delta_\tau + \sum_{\tau\in[t]}(\alpha_\tau-\alpha_{\tau-1})\langle \nabla f(\bx^*),\bw_\tau - \bw_{t+1}\rangle \\
    & \leq 2\eta L \sum_{\tau\in[t]} \alpha_{1:\tau}\Delta_\tau + \sum_{\tau\in[t]}(\alpha_\tau-\alpha_{\tau-1})\| \nabla f(\bx^*)\| \|\bw_\tau - \bw_{t+1}\| \\
    & \leq  \frac{1}{2T} \sum_{\tau\in[T]} \alpha_{1:\tau}\Delta_\tau + \alpha_tG^*D ~.
\end{align*}
Here, the first inequality employs the Young’s inequality.
For the second inequality, we introduce the notation \(\Delta_t := f(\bx_t) - f(\bx^*)\), and we follow Lemma \ref{lemma2}, which relates to the smoothness of the function $f$. In this step, we also set $\alpha_0 = 0$ and utilizes the property \(\alpha_\tau^2 \leq 2\alpha_{1:\tau}\), given that \(\alpha_\tau = \tau\). The third inequality uses the Cauchy-Schwarz inequality. The last inequality follows the assumption in Eq. \eqref{eq:bounded_diameter}. It uses the fact that \(t \leq T\) and $\Delta_t\geq0$, $\forall t$. This step also incorporates the choice of an appropriate learning rate parameter \(\eta\leq1/4LT\).

Plugging \textnormal{(A)} into Eq. \eqref{eq:regret-main}, gives us,
\begin{align}
\label{eq:reg-final}
    \alpha_{1:t}\Delta &\leq \frac{1}{2T} \sum_{\tau\in[T]} \alpha_{1:\tau}\Delta_\tau+ \frac{D^2}{2\eta} + \alpha_tG^*D +D\sum_{\tau\in[t]}\alpha_\tau\|{\varepsilon_\tau}\| ~.
\end{align}

\begin{lemma}[Lemma C.2 in \citet{levy2023mu}]
\label{lemma3}
let $\{A_t\}_{t\in[T]}, \{B_t\}_{t\in[T]}$ be sequences of non-negative elements, and assume that for any $t\leq T$,
\begin{equation*}
    A_t \leq B_T + \frac{1}{2T} \sum_{t\in[T]} A_t~.
\end{equation*}
Then the following bound holds,
\begin{equation*}
    A_T \leq 2B_T ~.
\end{equation*}
\end{lemma}

In the next step, let us define two terms:  $A_t := \alpha_{1:t}\E\left[f(\bx_t)-f(\bx^*)\right]$ and $B_t := \frac{D^2}{2\eta}+ \alpha_tG^*D  + D\sum_{\tau\in[t]}\alpha_\tau\E\|{\varepsilon_\tau}\|$. Note that the series $\{B_t\}_t$ forms a non-decreasing series of non-negative values, implying $B_t \leq B_T$ for any $t\in[T]$. As a result of Eq. \eqref{eq:reg-final}, we have that $A_t\leq B_T + \frac{1}{2T}\sum_{\tau\in[T]}A_\tau$. 

By employing Lemma \ref{lem:syncFilter}, we have,
\begin{align}
\label{eq:eps1}
   \sum_{\tau\in[T]}\alpha_\tau^2\E \| \varepsilon_\tau \|^2 &\leq\sum_{\tau\in[T]}\alpha^2_\tau\left({\frac{4\tsigma^2}{\tau m} + \frac{12c_\delta\tsigma^2}{\tau} + 6c_\delta\xi^2}\right) \nonumber \\ 
   & = \left(\frac{4\tsigma^2}{m} + {12c_\delta\tsigma^2}\right)\sum_{\tau\in[T]}{\tau}+ 6c_\delta\xi^2 \sum_{\tau\in[T]}{\tau^2}\nonumber \\ &\leq T^{2} \left(\frac{4\tsigma^2}{m} + {12c_\delta\tsigma^2}\right) + T^3 6c_\delta\xi^2.
\end{align}
Utilizing Jensen's inequality, enables us to establish,
\begin{align}
\label{eq:eps2}
    \sum_{\tau\in[T]}\alpha_\tau\E \| \varepsilon_\tau \| &= \sum_{\tau\in[T]}\sqrt{\left(\alpha_\tau\E \| \varepsilon_\tau \|\right)^2} \leq \sqrt{T}\sqrt{\sum_{\tau\in[T]}{\alpha_\tau^2\E \| \varepsilon_\tau \|^2} } \nonumber \\& \leq \sqrt{T} \sqrt{T^{2} \left(\frac{4\tsigma^2}{m} + {12c_\delta\tsigma^2}\right) + T^3 6c_\delta\xi^2} \nonumber \\ & \leq T^{1.5} \tsigma \sqrt{\frac{4}{m} + {12c_\delta}} + 6T^2\xi \sqrt{c_\delta}
    .
\end{align}
where the first inequality employs Eq. \eqref{eq:eps1} and the second inequality uses \(\sqrt{a+b}\leq \sqrt{a}+\sqrt{b}\) for non-negative \(a, b \in \reals\). The explanation behind this can be seen through the following steps: $\left(\sqrt{a}+\sqrt{b}\right)^2 = a + 2\sqrt{ab} + b \geq a + b$, whereby taking the square root of both sides of this equation, we obtain the desired inequality.

Leveraging Lemma \ref{lemma3} and Eq. \eqref{eq:eps2}, and acknowledging that \(\alpha_{1:T}=\Theta(T^2)\), as \(\alpha_t=t\), it follows that:
\begin{align}
\label{eq:lr}
    \E[f(\bx_T)-f(\bx^*)] &\leq \frac{2}{T^2}B_T \nonumber\\
    & = \frac{D^2}{T^2\eta}+ \frac{2G^*D}{T}   + \frac{2D }{T^2}\sum_{\tau\in[T]}\alpha_\tau\E\|{\varepsilon_\tau}\| \nonumber\\
    & \leq \frac{D^2}{T^2\eta}+ \frac{2G^*D}{T} + \frac{2D \tsigma}{\sqrt{T}}\sqrt{{12c_\delta} + \frac{4}{m}} + 12D\xi \sqrt{c_\delta}~.
\end{align}

Finally, choosing the optimal \(\eta \leq \frac{1}{4TL}\) gives us:
\begin{equation*}
    \E[f(\bx_T)-f(\bx^*)] \leq O\left(\frac{G^*D + LD^2}{ T}  +  \frac{D\tsigma}{\sqrt{T}}\sqrt{c_\delta + \frac{1}{m}} + \sqrt{c_\delta} D\xi \right)~.
\end{equation*} 
\end{proof}

\paragraph{Learning Rate Range for Optimal Convergence}
\label{par:lr}
To simplify our analysis, we assume \( G^* \) is negligible, which is a reasonable assumption. For example, if \( \bx^* = \arg\min_{\bx \in \K} f(\bx) = \arg\min_{\bx \in \reals^d} f(\bx) \), then \( \nabla f(\bx^*) = 0 \), resulting in \( G^* = 0 \). Considering that $G^*$ is negligible and using Eq. \eqref{eq:lr}, we derive the range for the learning rate \(\eta\) to maintain optimal convergence:
\begin{align*}
\frac{D^2}{T^2\eta} \leq \frac{2D \tsigma}{\sqrt{T}}\sqrt{12c_\delta + \frac{4}{m}} + 12D\xi \sqrt{c_\delta}~.
\end{align*}
Thus, the lower bound for \(\eta\) is:
\begin{align*}
\eta &\geq \frac{D^2}{T^2} \cdot \frac{1}{\frac{2D \tsigma}{\sqrt{T}}\sqrt{12c_\delta + \frac{4}{m}} + 12D\xi \sqrt{c_\delta}}
= \frac{D}{2\tsigma T^{1.5}\sqrt{12c_\delta + \frac{4}{m}} + 12\xi T^2}~.
\end{align*}
Given that \(\eta \leq \frac{1}{4TL}\), the range for \(\eta\) is:
\begin{align}
\eta \in \left[\frac{D}{2\tsigma T^{1.5}\sqrt{12c_\delta + \frac{4}{m}} + 12\xi T^2}, \frac{1}{4TL}\right]~.
\end{align}
This approach allows us to employ a broad spectrum of learning rates \(\eta \in [\eta_{min}, \eta_{max}]\), such that \(\eta_{max}/\eta_{min} = O(\sqrt{T})\) when considering the dominant parts at $T$. In contrast, the standard momentum method \cite{karimireddy2020Byzantine} requires a low-range of learning rates such that \(\eta_{max}/\eta_{min} = O(1)\) to achieve an order optimal bound. Specifically, for a starting point $\bx_0\in\real^d$ and $\bx^* := \arg\min_{\bx \in \real^d} f(\bx)$, the learning rate for momentum with adaptation to our notations is given by:
\begin{align*}
\eta_{momentum} \simeq \min\left\{O\left( \sqrt{\frac{L(f(\bx_0)-f(\bx^*))+c_\delta(\xi^2+\sigma^2)}{TL^2\sigma^2(\frac{1}{m}+c_\delta)}}\right), \frac{1}{8L}\right\}~.
\end{align*}

\subsubsection{Proof of Lemma ~\ref{lem:RegretBound}}

\begin{proof}[Proof of Lemma ~\ref{lem:RegretBound}]
The update rule $\bw_{\tau+1} = \Pi_\K (\bw_\tau - \eta\alpha_\tau\hat{\bd}_\tau)$ can be expressed as a convex optimization problem within the set $\K$:
\begin{align*}
    \bw_{\tau + 1} &= \Pi_{\K}\left({\bw_{\tau} - \eta \alpha_{\tau} \hat{\bd}_\tau}\right) 
    \\ &= \arg\min_{\bw\in\K} \|\bw_{\tau} - \eta \alpha_{\tau} \hat{\bd}_\tau - \bw\|^2 
    \\ &= \arg\min_{\bw\in\K}\{ \alpha_\tau \langle \hat{\bd}_\tau, \bw - \bw_\tau\rangle + \frac{1}{2\eta} \|\bw - \bw_\tau\|^2\}~.
\end{align*}

Here, the first equality is derived from the definition of the update rule, the second stems from the property of projection, and the final equality is obtained by reformulating the optimization problem in a way that does not affect the minimum value.

Given that \(\bw_{\tau+1}\) is the optimal solution of the above convex problem, by the optimality conditions, we have that:
    \begin{align*}
        \left\langle \alpha_\tau\hat{\bd}_\tau + \frac{1}{\eta}(\bw_{\tau+1} - \bw_{\tau}), \bw-\bw_{\tau+1}\right\rangle \geq 0, \quad \forall \bw\in\K~.
    \end{align*}

Rearranging this, summing over $t\geq1$ iterations, and taking $\bw=\bx^*$, we derive:
    \begin{align*}
        \sum_{\tau\in[t]}\alpha_\tau \langle \hat{\bd}_\tau, \bw_{\tau+1}-\bx^*\rangle & \leq \frac{1}{\eta} \sum_{\tau\in[t]}\langle \bw_\tau - \bw_{\tau+1}, \bw_{\tau+1}-\bx^*\rangle \\ &=  \frac{1}{2\eta}\sum_{\tau\in[t]} \left( \|\bw_\tau-\bx^*\|^2 - \|\bw_{\tau+1}-\bx^*\|^2 - \|\bw_\tau-\bw_{\tau+1}\|^2\right) \\ &=  \frac{1}{2\eta} \left( \|\bw_1-\bx^*\|^2 - \|\bw_{t+1}-\bx^*\|^2 - \sum_{\tau\in[t]}\|\bw_\tau-\bw_{\tau+1}\|^2\right) \\ & \leq \frac{D^2}{2\eta} -\frac{1}{2\eta}\sum_{\tau=1}^t \|\bw_{\tau}-\bw_{\tau+1}\|^2~.
    \end{align*}
The first equality equality is achieved through algebraic manipulation, and the last inequality follows the assumption in Eq. \eqref{eq:bounded_diameter}. 
\end{proof}

\subsubsection{Proof of Lemma ~\ref{lemma2}}
\begin{proof}[Proof of Lemma ~\ref{lemma2}]

\begin{lemma}[Lemma C.1 in \citet{levy2023mu}]
\label{lem:smooth-functions}
let $f:\real^d \rightarrow \real$ be an $L$-smooth function with a global minimum $\bw^*$, then for any $\bw\in\real^d$ we have,
\begin{equation*}
    \|{\nabla f(\bx)}\|^2 \leq 2L(f(\bx)-f(\bw^*))~.
\end{equation*}
\end{lemma}

     Let us define the function $h(\bx)=f(\bx)-f(\bx^*)-\langle \nabla f(\bx^*), \bx - \bx^* \rangle$. Given the convexity of $f(\bx)$, we have $f(\bx)-f(\bx^*) \geq \langle \nabla f(\bx^*), \bx - \bx^* \rangle\geq 0$, leading to $h(\bx)\geq 0$. As $h(\bx^*)=0$, this implies that $\bx^*$ is the global minimum of $h$. Applying Lemma \ref{lem:smooth-functions}, gives us,
    \begin{align*}
        \|{\nabla f(\bx)}-\nabla f(\bx^*)\|^2 = \|{\nabla h(\bx)}\|^2 \leq 2L(f(\bx)-f(\bx^*))~.
    \end{align*}
\end{proof}

\section{Robust Aggregators Analysis}

\begin{definition}[Rephrased from \citet{allouah2023fixing}] \textnormal{\( (\kappa, {\delta}) \)-\textbf{robustness}}.
        \label{def:kappa-robust}
        Let $\frac{|\B|}{m}={\delta}<1/2$ and $\kappa \geq 0$. An aggregation rule $\A$ is called $(\kappa, {\delta})$-robustness if for any $m$ vectors $\bx_1, \ldots, \bx_m$, and any set $S\subseteq[m]$ such that $|S|=|\GGG|$, we have,
        \begin{align*}
           \|{\hat{\bx} - \bar{\bx}_S} \|^2 \leq \frac{\kappa}{|S|} \sum_{i\in S} \|{\bx_i - \Bar{\bx}_S}\|^2 ~,
        \end{align*}
        where $\bar{\bx}_S:= \frac{1}{|S|} \sum_{i \in S} {\bx}_i$.
\end{definition}

\begin{lemma}
\label{lem:kappa-robust}
Let $\hat{\bx}$ be $(\kappa, {\delta})$-robustness where ${\delta}=\frac{|\B|}{m}$.  We define $(\kappa, {\delta})$-robustness as a rephrased version of the robust concept originally introduced in \citet{allouah2023fixing} and outlined in Definition \ref{def:kappa-robust}. Then, $\hat{\bx}$ is $(\kappa, {\delta})$-robust.

\begin{proof}
Given that $\hat{\bx}$ is $(\kappa, {\delta})$-robustness, by choosing $S=\GGG$, we can deduce that:
\begin{align*}
   \E\|{\hat{\bx} - \bar{\bx}_\GGG} \|^2 \leq \frac{\kappa}{|\GGG|} \sum_{i\in \GGG} \E\|{\bx_i - \Bar{\bx}_\GGG}\|^2 \leq \frac{\kappa}{|\GGG|} \sum_{i\in \GGG} \rho_i^2 = \kappa \rho^2~,
\end{align*}
where the second inequality follows the notations in Definition \ref{def1}.
\end{proof}
\end{lemma}

\section{$\mu^2$-SGD Overview}
Our approach adopts the \(\mu^2\)-SGD algorithm \citep{levy2023mu}, a novel method that combines two distinct momentum-based mechanisms for enhancing variance reduction.
\paragraph{AnyTime-SGD.}
The first mechanism originates from the AnyTime-SGD algorithm \citep{cutkosky2019anytime}. This algorithm employs a learning rate $\eta>0$ and a series of non-negative weights \(\{\alpha_t\}_t\). It operates by maintaining two series of query points, \(\{\bw_t\}_t\) and \(\{\bx_t\}_t\), and initializes \(\bx_1=\bw_1\). At each step \(t\), the algorithm first updates \(\bw_{t+1}\) in a manner akin to the standard SGD, but adjusted by a factor of \(\alpha_t\) for the gradient \(\bg_t:= \nabla f(\bx_t; \bz_t)\). Subsequently, it calculates the next query point \(\bx_{t+1}\) through a weighted average of the accumulated query points up to step \(t+1\), as shown in the following equation:
\begin{gather*}
    \bw_{t+1}=\Pi_{\K}\left({\bw_{t}- \eta\alpha_{t} \bg_{t}}\right), \quad
    \bx_{t+1}=\frac{1}{\alpha_{1:t+1}} \sum_{k\in[t+1]} \alpha_k \bw_k~.
\end{gather*}
A fundamental characteristic of the AnyTime-SGD is that the query point \(\bx_T\) also serves as the algorithm's output. This contrasts with standard SGD, which typically outputs the average of all query points or randomly selects one. Thus, at any given iteration \(t\), AnyTime-SGD consistently offers a potential solution for the optimization problem.
\paragraph{STORM.}
While the first mechanism employs weighted averaging of query points, the second leverages weighted averaging of the stochastic gradient estimators. This second mechanism is derived from the Stochastic Recursive Momentum (STORM) approach \citep{cutkosky2019momentum}. It utilizes a corrected momentum technique, which serves as an estimator for the actual gradient. This correction is implemented by adding a bias to the standard momentum equation, thereby refining the gradient estimates from previous iterations. This method is defined as follows:
\begin{equation*}
\bd_{t}=\nabla f(\bx_{t};\bz_{t}) + (1-\beta_{t})(\bd_{t-1} - \nabla f(\bx_{t-1};\bz_{t}))~.
\end{equation*}
A key feature of this approach is its ability to achieve implicit variance reduction through the use of the corrected momentum. This method not only adjusts the gradient estimations but also enhances the overall stability.
\paragraph{$\mu^2$-SGD.}  
The \(\mu^2\)-SGD algorithm integrates the AnyTime update step with the STORM gradient estimator. Instead of using the stochastic gradient $\bg_t$, it updates query points using a corrected momentum $\bd_t$. The update process is defined as follows:
\begin{gather}
    \label{eq:mu2sgd}
    \bw_{t+1}=\Pi_{\K}\left({\bw_{t}- \eta\alpha_{t} \bd_{t}}\right), \
    \bx_{t+1}=\frac{1}{\alpha_{1:t+1}} \sum_{k\in[t+1]} \alpha_k \bw_k, \
    \bx_1=\bw_1~.
\end{gather}
A primary aspect of the \(\mu^2\)-SGD algorithm is achieved by setting the momentum weights \(\{\beta_t:=1/t\}_t\) and the AnyTime weights \(\{\alpha_t:=t\}_t\). This enables $\mu^2$-SGD to significantly reduce the stochastic variance $\E\|\varepsilon_t\|^2 := \E\|\bd_t - \nabla f(\bx_t)\|^2 \leq O(\tsigma^2/t)$ at step $t$, where $\tsigma^2 \leq O(\sigma^2 + D^2 \sigma_L^2)$. This means that the variance decreases with each iteration, eventually becoming negligible, contrasting with standard SGD where the variance  $\E\|\varepsilon^{\textnormal{SGD}}_t\|^2 := \E\|\bg_t - \nabla f(\bx_t)\|^2$ remains uniformly bounded. Moreover, by choosing $\beta_t = 1/t$, $\mu^2$-SGD considers the entire gradient history \citep{karimireddy2021learning}, aligning with the AnyTime update's consideration of the entire history of query points. This approach differs from other momentum methods \citep{karimireddy2021learning, karimireddy2020Byzantine, allouah2023fixing}, where the momentum weight \(\beta_t\) is tied to the learning rate \(\eta_t\), thereby limiting the gradient history consideration to just \(\sqrt{t}\) gradients. 

\section{Experiments}
\label{app:exp}
\subsection{Technical Details}
\begin{remark}
    Recall that the AnyTime update step of \(\bx_t\) is defined as: 
    \begin{align*}
        \bx_t := \frac{\alpha_t \bw_t + \alpha_{1:t-1}\bx_{t-1}}{\alpha_{1:t}}.
    \end{align*} 
    This formulation provides an alternative representation of the update step, which can be expressed as:
    \begin{align*}
        \bx_{t} = \gamma_t \bw_t + (1 - \gamma_t) \bx_{t-1},
    \end{align*}
    where \(\gamma_t := \frac{\alpha_{t}}{\alpha_{1:t}}\). Furthermore, by setting \(\alpha_t = C\alpha_{1:t-1}\) for a constant \(C>0\), it follows that \(\gamma_t = \frac{{C\alpha_{1:t-1}}}{C\alpha_{1:t-1}+\alpha_{1:t-1}} = \frac{C}{C + 1}\) is a constant for every \(t \geq 1\).
\end{remark}

\subsubsection{Datasets}

We evaluated our approach on two benchmark datasets: CIFAR-10 \citep{krizhevsky2014cifar} and MNIST \citep{lecun2010mnist}.

\begin{itemize}
    \item \textbf{CIFAR-10}: This dataset consists of 60,000 32x32 color images in 10 classes, with 6,000 images per class. The dataset is divided into 50,000 training images and 10,000 testing images.
    \item \textbf{MNIST}: This dataset consists of 70,000 28x28 grayscale images of handwritten digits in 10 classes, with 60,000 training images and 10,000 testing images.
\end{itemize}

\subsubsection{Model Architectures}

To demonstrate the efficiency of momentum parameters that account for the entire history, we conducted experiments using simple convolutional networks. Our choice of simple conv networks was driven by the need to mitigate the sensitivity to numerical errors, which is often encountered in more complex models with very small momentum parameters. Additionally, we extended our evaluation to a more complex scenario by using the CIFAR-10 dataset with a ResNet18 model, employing fixed momentum parameters for comparison.

\begin{itemize}
    \item \textbf{Simple Conv Network for MNIST}:  This model consists of two convolutional layers with batch normalization and ReLU activation, followed by max pooling layers. The first convolutional layer has 16 filters, and the second layer has 32 filters. The output from the convolutional layers is flattened and passed through a fully connected layer with 1,568 units for classification.
    \item \textbf{Simple Conv Network for CIFAR-10}: This model also consists of two convolutional layers with batch normalization and ReLU activation, followed by max pooling layers. The first convolutional layer has 16 filters, and the second layer has 32 filters. The output from the convolutional layers is flattened and passed through a fully connected layer with 2,048 units for classification.
    \item \textbf{ResNet18 for CIFAR-10}: We used the standard ResNet18 architecture, which includes multiple residual blocks with convolutional layers, batch normalization, and ReLU activation. This model is designed to handle more complex image classification tasks.
\end{itemize}

\subsubsection{Algorithms and Baselines}

We evaluate our proposed \(\mu^2\)-SGD algorithm by comparing it with the standard momentum method as described by \citet{karimireddy2021learning}. The specific momentum parameters utilized in our experiments are detailed below:

\begin{itemize}
    \item \textbf{Standard Momentum}: We employ \(\beta_t = 0.9\), as suggested by \citet{karimireddy2021learning}. 
    
    \item \textbf{\(\mu^2\)-SGD}: We experiment with two distinct parameter settings:
        \begin{itemize}
            \item Dynamic parameters: \(\alpha_t = t\) and \(\beta_t = 1/t\), in line with our theoretical suggestion.
            \item Fixed parameters: \(\gamma_t = 0.1\) and \(\beta_t = 0.9\), more similar to the fixed standard momentum parameter. 
        \end{itemize}
        Detailed configurations are provided in Section \ref{sec:conf}.
\end{itemize}

\subsubsection{Evaluation Metrics}

To assess the performance of our algorithm, we used the following metric:

\begin{itemize}
    \item \textbf{Accuracy}: The proportion of correctly classified instances over the total instances.
\end{itemize}

\subsubsection{Experimental Setup}
\label{sec:conf}

We conducted a hyperparameter search to identify the optimal settings for our experiments. 

\textbf{Learning Rate}: We experimented with a range of learning rates from $10^{-4}$ to $10^{1}$. For experiments requiring a single learning rate, we selected 0.1, which was found to be optimal within this range.

The table below summarizes the configurations used in our experiments, including the settings for $\alpha_t$, $\beta_t$, and $\gamma_t$ for the $\mu^2$-SGD algorithm, as well as the dataset, model, batch size, and gradient clipping values to enhance performance, as implemented in \citet{allouah2023fixing}. Note that gradient clipping was not applied in the experiments involving a wide range of learning rates due to its impact on the size of the update step.

\begin{table}[h]
\centering
\begin{tabular}{cccccccc}
\toprule
\textbf{Configuration} & \textbf{$\alpha_t$} & \textbf{$\beta_t$} & \textbf{$\gamma_t$} & \textbf{Dataset} & \textbf{Model} & \textbf{Batch Size} & \textbf{Gradient Clipping} \\ \midrule
\textbf{1} & $t$ & $1/t$ & - & MNIST & Simple Conv & 4 & 2 \\ 
\textbf{2} & $t$ & $1/t$ & - & CIFAR-10 & Simple Conv & 64 & 5 \\ 
\textbf{3} & - & $0.9$ & $0.1$ & MNIST & Simple Conv & 64 & - \\ 
\textbf{4} & - & $0.9$ & $0.1$ & CIFAR-10 & ResNet18 & 8 & - \\ 
\bottomrule
\end{tabular}
\caption{Experimental Configurations}
\label{tab:configurations}
\end{table}

\subsubsection{Attacks}

To evaluate the robustness of our algorithms, we tested them against the following adversarial attacks:
\begin{itemize}
    \item \textbf{Label-Flipping}  \citep{allen2020Byzantine}: This attack flips the original target labels to incorrect labels by subtracting the original label from 9,
    \[
     \text{flipped\_label} = 9 - \text{original\_label}.
     \]
    \item \textbf{Sign-Flipping}  \citep{allen2020Byzantine}: This attack flips the signs of the momentums, in the spirit of faults that occur when bits are transmitted incorrectly from workers to the central server.
    \[
     \text{Byzantine\_update} = -\text{worker\_momentum}.
     \]
    \item \textbf{A Little Is Enough} \citep{baruch2019little}: This attack is designed to lie under the stochastic "noise blanket." It calculates the maximum allowable deviation \(z_{\text{max}}\) based on the number of honest workers, then perturbs the honest updates by subtracting the product of the standard deviation and \(z_{\text{max} }\) from the mean of the honest updates.
    \[
     \text{Byzantine\_update} = \text{mean}(\text{honest\_momentums}) - \text{std}(\text{honest\_momentums}) \cdot z_{\text{max} }.
     \]
     \item \textbf{Empire} \citep{xie2020fall}: This attack scales the mean of the honest momentums by a small factor \(\epsilon\) in the negative direction, 
     \[
     \text{Byzantine\_update} = -\epsilon \cdot \text{mean}(\text{honest\_momentums}).
     \]
      where the mean and standard deviation are calculated coordinate-wise, and we set \(\epsilon = 0.5\).
\end{itemize}

\subsubsection{Implementation Details}

The implementation was carried out using PyTorch. The code was written in Python and executed on NVIDIA A30 GPU for MNIST and NVIDIA GeForce RTX 3090 GPU for CIFAR-10. All experiments were repeated three times with different random seeds to ensure statistical significance, and the results reported are the averages of these runs.

\subsection{Complete Experimental Results}

\subsubsection{MNIST}

\begin{figure}[ht]
\vskip 0.2in
\begin{center}
\centerline{\includegraphics[width=0.8\columnwidth]{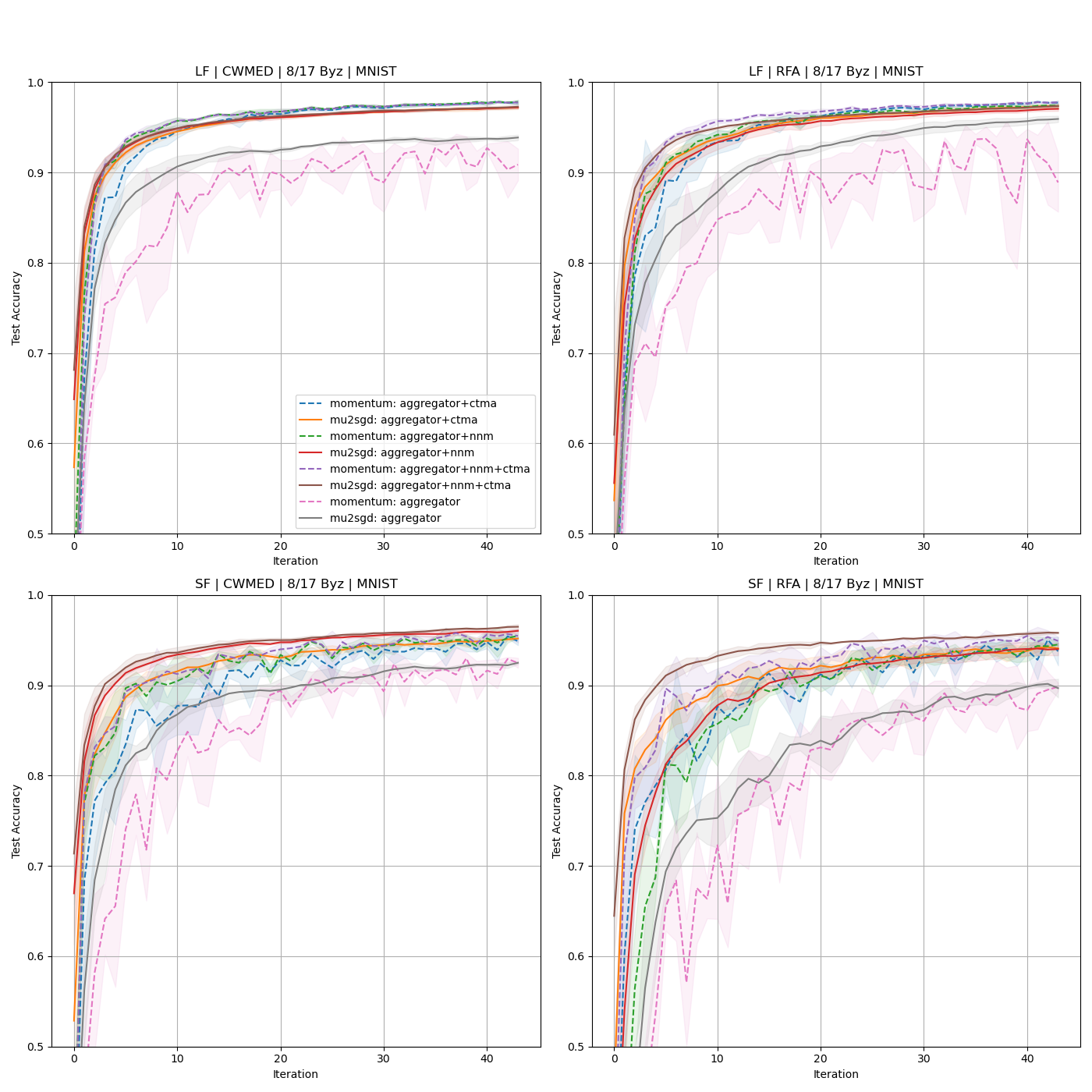}}
\caption{Performance comparison of CTMA with existing meta-aggregators, and $\mu^2$-SGD with momentum under \textit{sign-flipping} and \textit{label-flipping} attacks with 8/17 Byzantine workers on the MNIST dataset (Conf. 1 in Table \ref{tab:configurations}). Here, we observe that CTMA enhances the performance of both $\mu^2$-SGD and momentum. Furthermore, CTMA performs at least as well as, if not better than, other meta-aggregators. The integration of nnm with CTMA can further improve performance. Notably, $\mu^2$-SGD demonstrates high stability compared to momentum, even without the assistance of any meta-aggregator. This stability is valuable for increasing resilience against heavy attacks and boosting overall performance.}
\label{fig:lf+sf-8}
\end{center}
\vskip -0.2in
\end{figure}

\begin{figure}[ht]
\vskip 0.2in
\begin{center}
\centerline{\includegraphics[width=0.8\columnwidth]{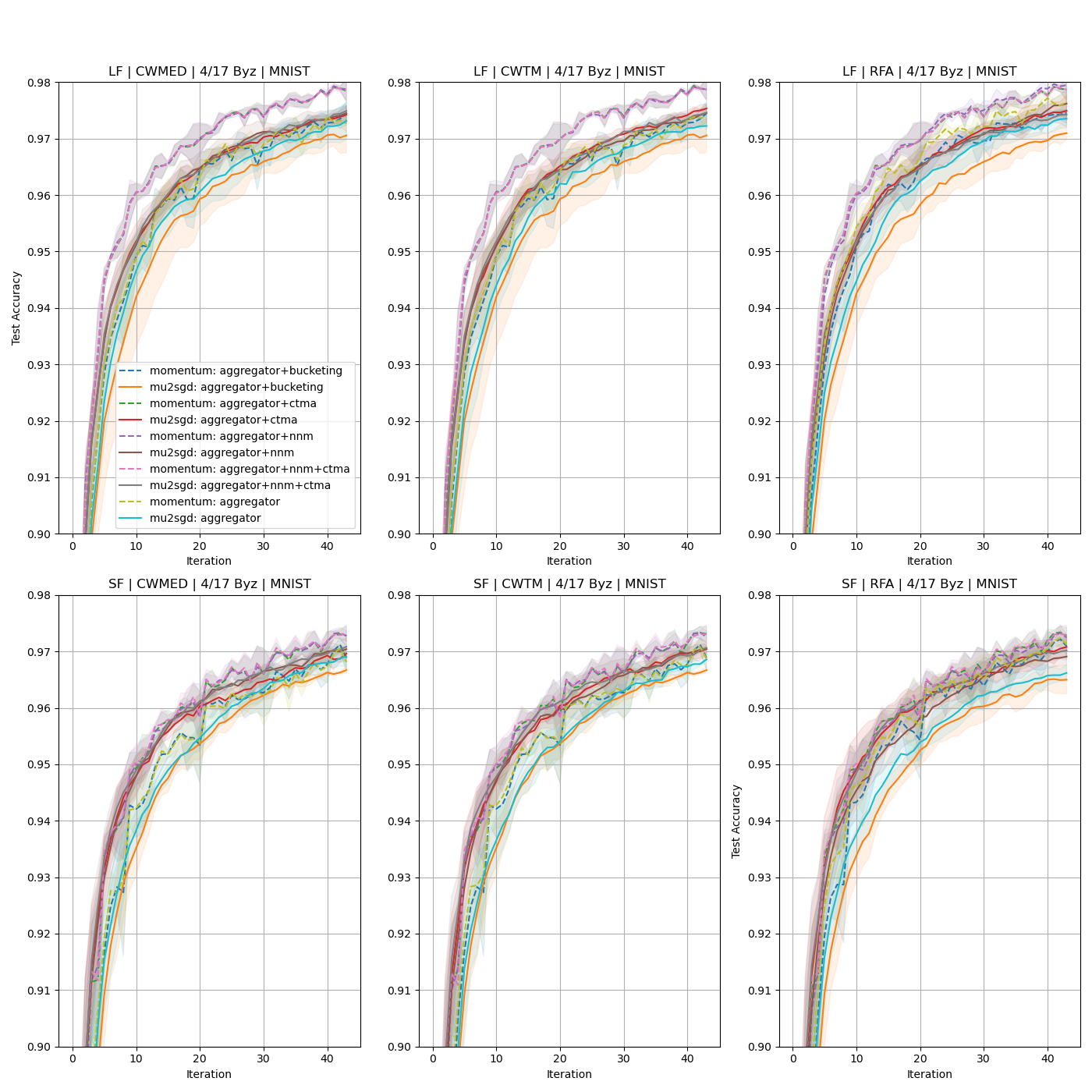}}
\caption{Performance comparison of CTMA with existing meta-aggregators, and $\mu^2$-SGD with momentum under a weaker \textit{sign-flipping} and \textit{label-flipping} attack with 4/17 Byzantine workers on the MNIST dataset (Conf. 1 in Table \ref{tab:configurations}). In this scenario, even though the performance of momentum is noisier over iterations compared to $\mu^2$-SGD, it does not require additional stability to perform effectively and outperforms $\mu^2$-SGD. CTMA enhances the performance of both $\mu^2$-SGD and momentum, maintaining consistent improvement as observed under the heavier attacks shown in Figure \ref{fig:lf+sf-8}.}
\label{fig:mu2sgd-vs-momentum}
\end{center}
\vskip -0.2in
\end{figure}

\begin{figure}[ht]
\vskip 0.2in
\begin{center}
\centerline{\includegraphics[width=0.8\columnwidth]{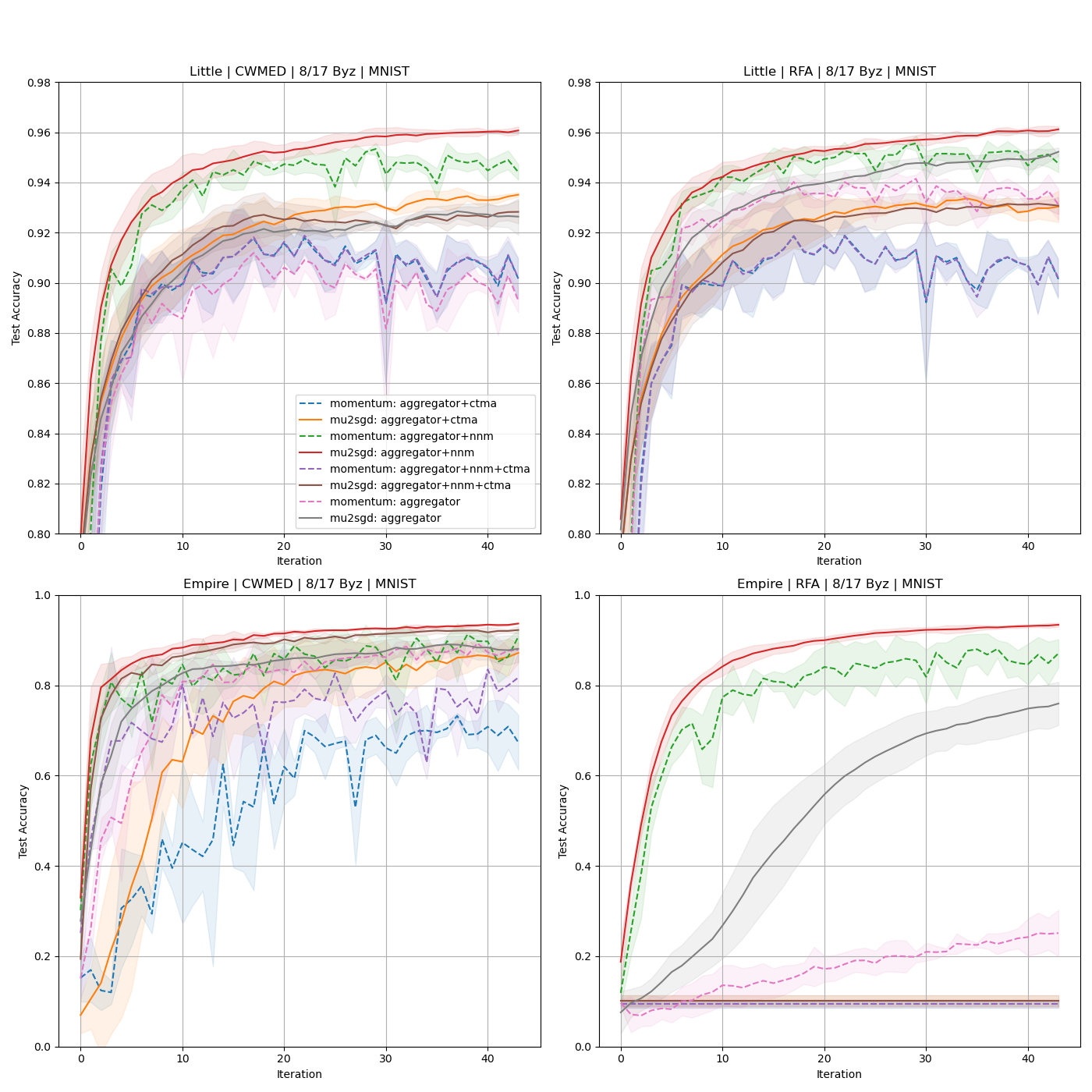}}
\caption{Performance comparison of CTMA with existing meta-aggregators, and $\mu^2$-SGD with momentum under SOTA low-variance attacks, \textit{empire} and \textit{little}, with 8/17 Byzantine workers on the MNIST dataset (Conf. 1 in Table \ref{tab:configurations}). These low-variance attacks are harder to detect and represent an especially severe attack scenario. In this context, CTMA performs poorly due to its strong reliance on the variance among the workers' outputs. In contrast, NNM performs very effectively for both momentum and $\mu^2$-SGD. The low variance in $\mu^2$-SGD enhances the effectiveness of NNM, making $\mu^2$-SGD more robust and particularly valuable against heavy low-variance attacks compared to momentum, with or without the addition of NNM.}
\label{fig:ctma-fails}
\end{center}
\vskip -0.2in
\end{figure}

\begin{figure}[ht]
\vskip 0.2in
\begin{center}
\centerline{\includegraphics[width=0.8\columnwidth]{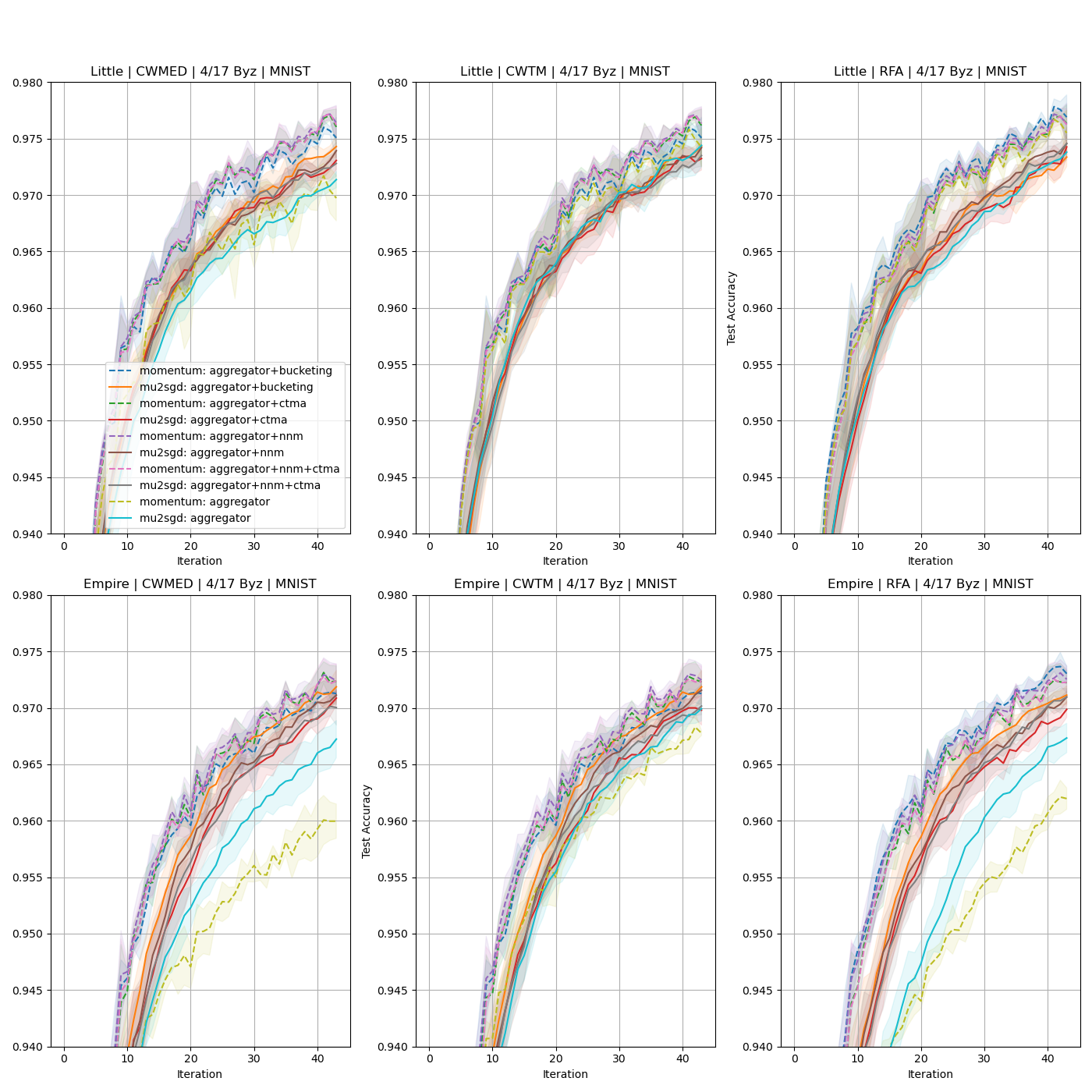}}
\caption{Performance comparison of CTMA with existing meta-aggregators, and $\mu^2$-SGD with momentum under SOTA low-variance attacks, \textit{empire} and \textit{little}, with 4/17 Byzantine workers on the MNIST dataset (Conf. 1 in Table \ref{tab:configurations}).  These are low-variance attacks in a weaker scenario with 4 Byzantine workers. In this context, CTMA performs well and often compares favorably to other meta-aggregators, enhancing the performance of both $\mu^2$-SGD and momentum. Furthermore, for the \textit{empire} attack, the $\mu^2$-SGD outperforms momentum without the addition of a meta-aggregator. However, when a meta-aggregator is added, the inherent stability of $\mu^2$-SGD becomes less beneficial, and momentum tends to outperform it.}
\label{fig:ctma-low-mnist}
\end{center}
\vskip -0.2in
\end{figure}

\begin{figure}[ht]
\vskip 0.2in
\begin{center}
\centerline{\includegraphics[width=0.8\columnwidth]{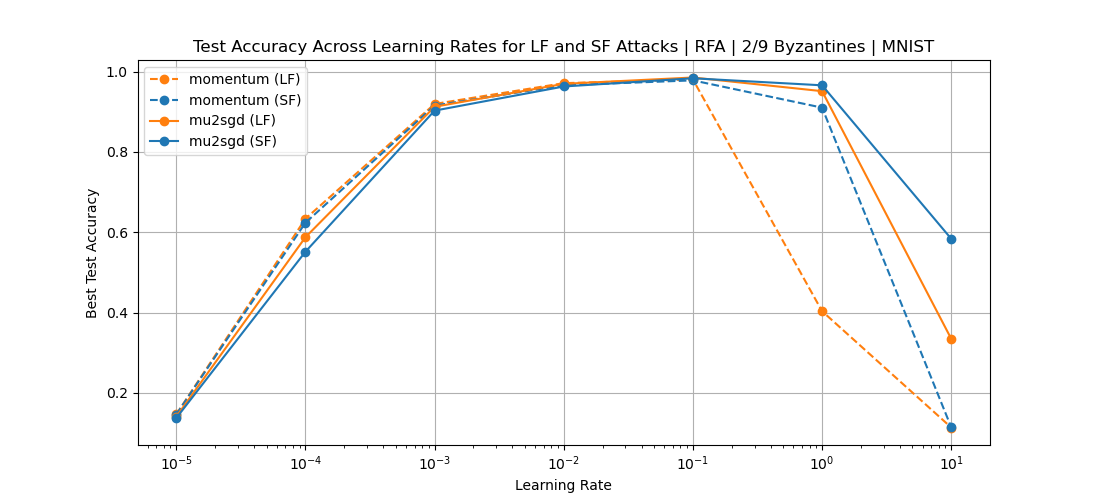}}
\caption{Performance comparison of $\mu^2$-SGD with momentum across a wider range of learning rates under \textit{label-flipping} and \textit{sign-flipping} attacks with 2/9 Byzantine workers on the MNIST dataset (Conf. 3 in Table \ref{tab:configurations}).}
\end{center}
\vskip -0.2in
\end{figure}

\clearpage
\subsubsection{CIFAR-10}

\begin{figure}[ht]
\vskip 0.2in
\begin{center}
\centerline{\includegraphics[width=0.8\columnwidth]{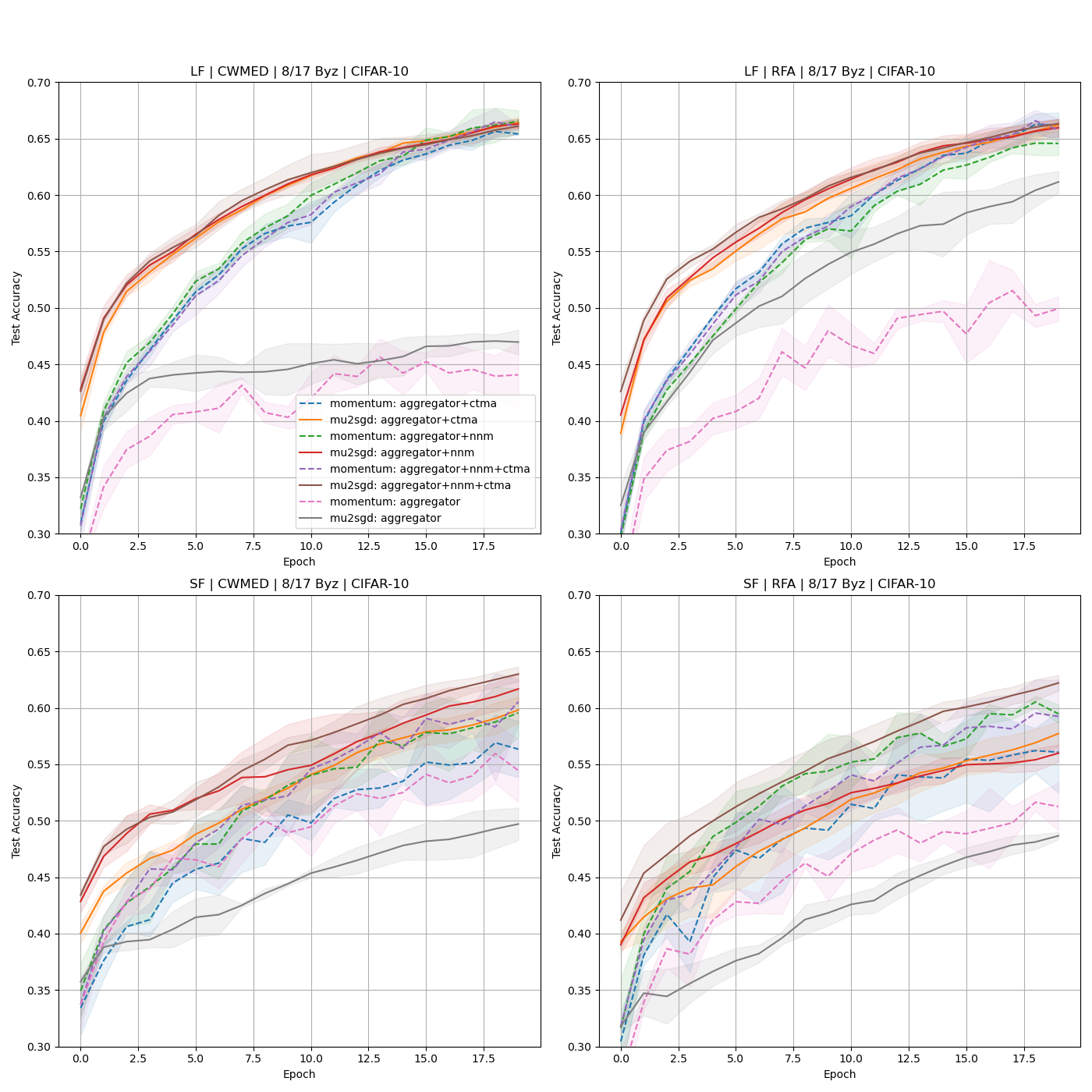}}
\caption{Performance comparison of CTMA with existing meta-aggregators, and $\mu^2$-SGD with momentum under \textit{sign-flipping} and \textit{label-flipping} attacks with 8/17 Byzantine workers on the CIFAR-10 dataset (Conf. 2 in Table \ref{tab:configurations}).}
\end{center}
\vskip -0.2in
\end{figure}

\begin{figure}[ht]
\vskip 0.2in
\begin{center}
\centerline{\includegraphics[width=0.8\columnwidth]{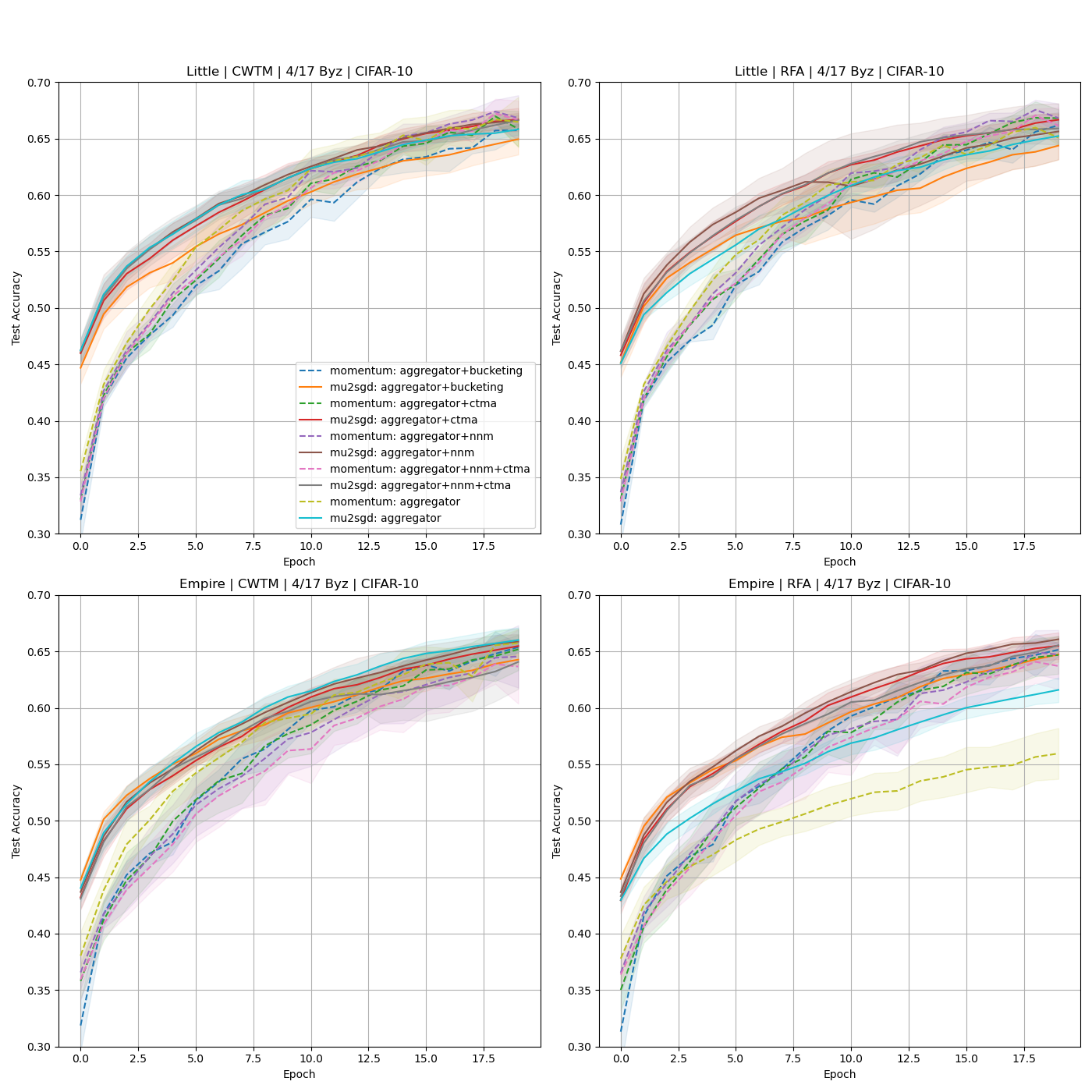}}
\caption{Performance comparison of CTMA with existing meta-aggregators, and $\mu^2$-SGD with momentum under SOTA low-variance attacks, \textit{empire} and \textit{little}, with 4/17 Byzantine workers on the CIFAR-10 dataset (Conf. 2 in Table \ref{tab:configurations}).}
\label{fig:ctma-low-cifar}
\end{center}
\vskip -0.2in
\end{figure}

\begin{figure}[ht]
\vskip 0.2in
\begin{center}
\centerline{\includegraphics[width=0.8\columnwidth]{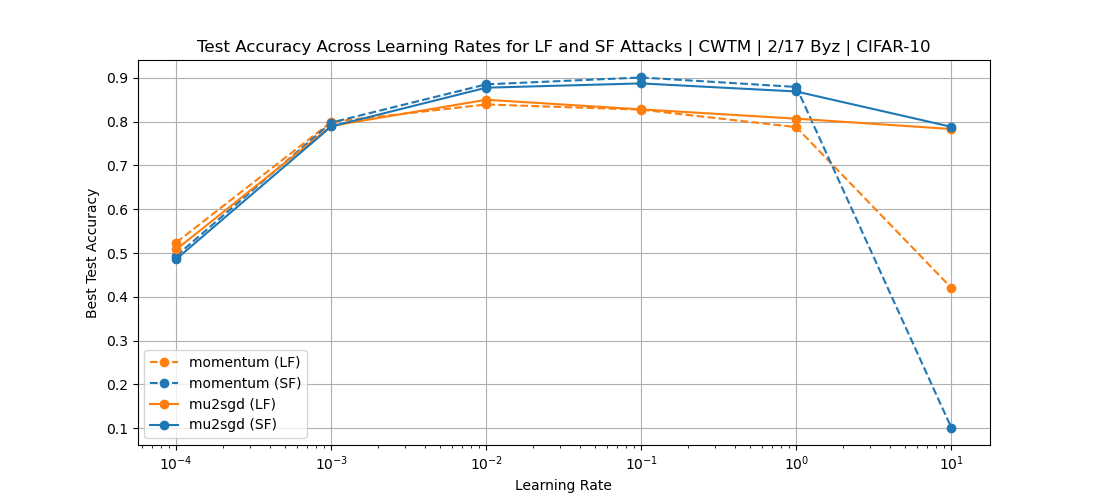}}
\caption{Performance comparison of $\mu^2$-SGD with momentum across a wider range of learning rates under \textit{label-flipping} and \textit{sign-flipping} attacks with 2/17 Byzantine workers on the CIFAR-10 dataset (Conf. 4 in Table \ref{tab:configurations}).}
\end{center}
\vskip -0.2in
\end{figure}
\end{document}

%% file: defs.tex
\newtheorem{theorem}{Theorem}[section]
\newtheorem{definition}{Definition}[section]
\newtheorem{lemma}{Lemma}[section]
\newtheorem{remark}{Remark}[section]

\def\reals{\mathbb{R}}
\newcommand{\real}{\mathbb{R}}

\newcommand{\E}{\mathbb{E}}
\newcommand{\tsigma}{\tilde{\sigma}}

\newcommand{\sigmal}{\sigma_L}

\newcommand{\A}{\mathcal{A}}
\newcommand{\B}{\mathcal{B}}
\newcommand{\D}{\mathcal{D}}
\newcommand{\G}{\mathcal{G}}
\newcommand{\K}{\mathcal{K}}
\newcommand{\F}{\mathcal{F}}
\newcommand{\MMM}{\mathcal{M}}

\newcommand{\GGG}{\mathcal{G}}

\newcommand{\ba}{\mathbf{a}}
\newcommand{\bb}{\mathbf{b}}
\newcommand{\bc}{\mathbf{c}}
\newcommand{\bd}{\mathbf{d}}

\newcommand{\bg}{\mathbf{g}}

\newcommand{\bw}{\mathbf{w}}
\newcommand{\bx}{\mathbf{x}}
\newcommand{\by}{\mathbf{y}}
\newcommand{\bz}{\mathbf{z}}